\documentclass{article}

\usepackage{arxiv}

\usepackage[utf8]{inputenc} %
\usepackage[T1]{fontenc}    %
\usepackage{hyperref}       %
\usepackage{url}            %
\usepackage{booktabs}       %
\usepackage{amsfonts}       %
\usepackage{amsmath}
\usepackage{amssymb}
\usepackage{amsthm}
\usepackage{nicefrac}       %
\usepackage{microtype}      %
\usepackage{cleveref}       %
\usepackage{graphicx}
\usepackage{natbib}
\usepackage{doi}
\usepackage{IEEEtrantools}
\usepackage{xcolor}
\usepackage{makecell}
\usepackage{tikz}
\usepackage{multirow}
\usepackage[shortlabels]{enumitem}
\setlist[enumerate]{nosep}
\setlist[itemize]{nosep}
\usepackage{mathtools}
\usepackage{subcaption}
\usepackage{thmtools}
\usepackage{thm-restate}

\usepackage[page]{appendix}
\usepackage{etoolbox}

\usepackage{titlesec}

\titleformat{\paragraph}[runin]
  {\normalfont\normalsize\bfseries}
  {}
  {0pt}
  {}
  [.]

\renewcommand{\appendixtocname}{Appendix Contents.}

\makeatletter
\let\oldappendix\appendices

\renewcommand{\appendices}{%
  \clearpage
  \renewcommand{\thesection}{\Roman{section}}
  \let\tf@toc\tf@app
  \addtocontents{app}{\protect\setcounter{tocdepth}{2}}
  \immediate\write\@auxout{%
    \string\let\string\tf@toc\string\tf@app^^J
  }
  \oldappendix
}%

\newcommand{\listofappendices}{%
  \begingroup
  \renewcommand{\contentsname}{\appendixtocname}
  \let\@oldstarttoc\@starttoc
  \def\@starttoc##1{\@oldstarttoc{app}}
  \tableofcontents%
  \endgroup
}
\makeatother

\theoremstyle{plain}
\newtheorem{theorem}{Theorem}[section]

\newtheorem{lemma}[theorem]{Lemma}
\newtheorem{corollary}[theorem]{Corollary}
\theoremstyle{definition}
\newtheorem{definition}[theorem]{Definition}

\newtheorem{remark}[theorem]{Remark}

\newcommand{\bbE}{\mathbb{E}}
\newcommand{\calK}{\mathcal{K}}
\newcommand{\calR}{\mathcal{R}}
\newcommand{\calY}{\mathcal{Y}}
\newcommand{\calX}{\mathcal{X}}
\newcommand{\calZ}{\mathcal{Z}}
\newcommand{\calS}{\mathcal{S}}
\newcommand{\calB}{\mathcal{B}}
\newcommand{\calG}{\mathcal{G}}
\newcommand{\calU}{\mathcal{U}}

\DeclareMathOperator{\argmin}{argmin}
\DeclareMathOperator{\softmax}{softmax}
\DeclareMathOperator{\logsumexp}{logsumexp}
\DeclareMathOperator{\Risk}{Risk}
\DeclareMathOperator{\err}{err}

\providecommand{\diff}{\,\mathrm{d}}

\newcommand{\Ncal}{N_{\mathrm{cal}}}

\newcommand{\calL}{\mathcal{L}}
\newcommand{\bbR}{\mathbb{R}}

\newif\ifarxiv
\arxivtrue

\title{Rethinking Early Stopping: Refine, Then Calibrate}

\date{}

\usepackage{authblk}

\author[1,2]{Eugène Berta\thanks{\texttt{eugene.berta@inria.fr}}}
\author[1,2]{David Holzmüller}
\author[1,3]{Michael I. Jordan}
\author[1,2]{Francis Bach}

\affil[1]{Inria, Paris}
\affil[2]{École Normale Supérieure, PSL Research University, Paris}
\affil[3]{University of California, Berkeley}

\renewcommand{\headeright}{}
\renewcommand{\undertitle}{\vspace{-1em}}
\renewcommand{\shorttitle}{Rethinking Early Stopping: Refine, Then Calibrate}

\hypersetup{
pdftitle={Rethinking Early Stopping: Refine, Then Calibrate},
pdfsubject={cs.LG, stat.ML},
pdfauthor={Eugène Berta, David Holzmüller, Michael I. Jordan, Francis Bach},
pdfkeywords={Classification, Early stopping, Calibration, Proper scoring rules},
}

\begin{document}
\maketitle

\begin{abstract}
Machine learning classifiers often produce probabilistic predictions that are critical for accurate and interpretable decision-making in various domains.
The quality of these predictions is generally evaluated with proper losses, such as cross-entropy, which decompose into two components: \emph{calibration error} assesses general under/overconfidence, while \emph{refinement error} measures the ability to distinguish different classes.
In this paper, we present a novel variational formulation of the calibration-refinement decomposition that sheds new light on post-hoc calibration, and enables rapid estimation of the different terms.
Equipped with this new perspective, we provide theoretical and empirical evidence that calibration and refinement errors are not minimized simultaneously during training.
Selecting the best epoch based on validation loss thus leads to a compromise point that is suboptimal for both terms.
To address this, we propose minimizing refinement error only during training (\emph{Refine,...}), before minimizing calibration error post hoc, using standard techniques (\emph{...then Calibrate}).
Our method integrates seamlessly with any classifier and consistently improves performance across diverse classification tasks.
\end{abstract}


\section{Introduction}

Accurate classification lies at the heart of many machine learning applications, from medical diagnosis and fraud detection to autonomous driving and weather forecasting.
Modern classifiers predict not only a class label but also a probability vector reflecting the model's confidence that the instance belongs to each class.
These probabilistic outputs are essential for informed downstream decision-making, especially in high-stakes scenarios where uncertainty must be explicitly accounted for.

However, raw probability estimates produced by learning algorithms are frequently poorly calibrated, meaning that predicted scores do not align with real-world probabilities, leading to overconfident or underconfident predictions that can undermine reliability and trustworthiness.
This issue is well-documented in classical machine learning; models such as support vector machines, boosted decision trees, and standard multi-layer perceptrons typically require post-hoc calibration to yield reliable probabilities \citep{platt1999probabilistic, zadrozny2001obtaining, zadrozny2002transforming, niculescu2005predicting}.

In the context of deep learning, \citet{guo2017} famously observed that deep neural networks of their time were becoming increasingly overconfident and mis-calibrated compared to their predecessors.
More recently, literature focusing on modern architectures has nuanced this narrative: it appears that in the regime of over-parameterization and single-pass pre-training on large datasets, models often exhibit good calibration without intervention \citep{desai2020calibration, minderer2021revisiting, achiam2023gpt}.

While these recent findings are promising for pre-trained foundation models, miscalibration remains a pervasive issue in multi-epoch training.
This setting is ubiquitous in standard learning scenarios, particularly for tabular data, and remains the standard mechanism for fine-tuning modern architectures to specific tasks.
Recent work has shown that even well-calibrated pre-trained models tend to lose this property during fine-tuning, as the iterative optimization required to specialize the model reintroduces the risk of overfitting \citep{ranjan2024post, wang2025towards, xiao2025restoring}.

This issue has mostly been considered through the lens of post-hoc calibration, which consists in adjusting the predicted probabilities after training, using a small held-out portion of the dataset.
An overlooked aspect of the problem is that the calibration error contributes to the total loss of a classifier \citep{brocker2009}.
In \Cref{figure:TrainingCurve}, we plot the decomposition into calibration error and refinement error of the classification loss (here cross-entropy) on the validation set during training for a deep neural network on a computer vision task. We observe that the two errors are not minimized simultaneously. The resulting loss minimizer thus has a non-zero calibration error and is suboptimal for refinement.

This accords with a recent empirical observation of \citet{ranjan2024post}, who noted that selecting the training epoch at which to stop based on classification loss after post-hoc transformations yields large performance gains in the classical setting of deep neural network multi-epoch training for computer vision but also in the context of large language model (LLM) instruction tuning.
In this paper, we provide a general understanding of this phenomenon by introducing a novel variational formulation of the calibration-refinement decomposition. Our analysis reveals that there is room for improvement in the way we train probabilistic classifiers in general.

\begin{figure}[ht]
\begin{center}
\centerline{\includegraphics[width=0.5\textwidth]{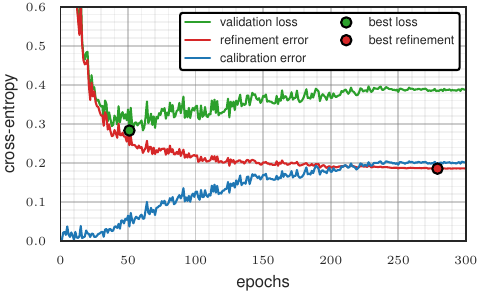}}
\caption{\emph{Calibration-refinement decomposition during training.} We plot the calibration (blue, bottom curve) and refinement (red, middle curve) terms of the validation cross-entropy loss (green, top curve), measured with our estimator from \Cref{section:OurMethod}, for a deep neural network (ResNet-18) trained on a computer vision dataset (CIFAR-10). Calibration and refinement errors are not minimized simultaneously, the loss minimizer (green dot) is thus suboptimal for both. See \Cref{appendix:ComputerVision} for details regarding the training procedure.}
\label{figure:TrainingCurve}
\end{center}
\end{figure}

In \Cref{section:LossDecomposition}, we introduce novel variational formulations of the calibration, refinement, and sharpness terms that appear in classical proper loss decompositions.
We show that the refinement error of a probabilistic classifier $f$ can be written as $\min_{g} \text{Risk}(g \circ f)$, where the minimum is taken over all measurable functions.
Since risk is calibration plus refinement, the calibration error is $\text{Risk}(f) - \min_{g} \text{Risk}(g \circ f)$.
This clarifies post-hoc calibration: selecting the function $g^*$ that minimizes the empirical risk of $g \circ f$ in a class of functions $g \in \mathcal{G}$ is a principled way to remove the calibration error.
The resulting risk after recalibration is a simple refinement error estimator while the risk difference before and after recalibration estimates calibration error.
The contributions of calibration and refinement errors to the total classification loss can thus be estimated simply by fitting a recalibration function.

In \Cref{section:OurMethod} we propose using the loss after temperature scaling as a refinement estimator, which applies to both binary and multi-class classifiers.
Tracking calibration and refinement errors during training requires fitting a recalibration function after each epoch, which we make practical via an efficient temperature scaling implementation that yields a fast proxy for minimization over all measurable functions.
Using this, we observe empirically for a variety of machine learning models that calibration and refinement errors are not minimized simultaneously (as in \Cref{figure:TrainingCurve}).
Early stopping on the validation loss minimizer therefore does not minimize either error.
To address this limitation of the standard training paradigm, we propose minimizing only the refinement error during training, preventing the variations in calibration error from contaminating our early stopping metric, and taking advantage of the fact that it can be minimized after training.

In \Cref{section:Experiments}, we show empirically that early stopping based on refinement error, as estimated by fitting a simple recalibration function after each iteration, yields smaller loss at test time across various architectures and data modalities.
Far from restricting classification performance, carefully dealing with refinement and calibration errors allows for smaller loss at test time, whatever the proper score of interest.

We also initiate the theoretical study of this phenomenon in Sections~\ref{section:GaussianDataModel}~and~\ref{section:HighDimensionalLogisticRegression}, focusing on the simplified setting of high-dimensional regularized logistic regression.
In this setting, we demonstrate theoretically that calibration and refinement errors are minimized for different values of the regularization parameter. Refinement-based early stopping thus provably achieves significant loss reduction by minimizing both refinement error (during training) and calibration error (post hoc).
Our mathematical model allows us to study the impact of problem parameters on calibration and refinement minimizers, and we uncover the existence of different regimes depending on the nature of the problem.

The estimated refinement error can also be used for tuning hyper-parameters, or as a stopping metric in any iterative machine learning algorithm, making our method broadly applicable.
To facilitate practical adoption, we make our refinement estimator available in the open source \texttt{probmetrics} package\footnote{\url{https://github.com/probkit/probmetrics}}.

\section{Related work}

\paragraph{Proper scoring rules}
Proper scoring rules are fundamental tools for assessing the quality of probabilistic forecasts. For a comprehensive overview, see \citet{gneiting2007strictly}.
From a decision-theoretic perspective, proper scoring rules arise naturally: the loss incurred by choosing the Bayes optimal action under the predictive distribution is an instance of a proper scoring rule \citep{grunwald2004game}.

\paragraph{Decompositions of proper scores} 
For the Brier score \citep{brier1950verification}, the calibration-refinement decomposition~(\ref{equation:DecompositionBrocker}) was introduced by \citet{sanders1963subjective} and the calibration-sharpness decomposition (\ref{equation:DecompositionBrockerSharpness}) by \citet{murphy1973new}.
\citet{brocker2009} extended these ideas to general proper scoring rules.
\citet{kull2015novel} reformulated these decompositions purely in terms of divergences between different random variables and introduced additional components such as irreducible, epistemic, and grouping loss.
Alternative formulations have also been proposed by \citet{pohle2020murphy}.
However, none of these prior works provide a variational interpretation of such decompositions.

\paragraph{Calibration}
In the machine learning literature, calibration has been an active research area after it became clear that classical architectures like support vector machines, tree-based models, and small MLPs do not yield reliable probability estimates out of the box \citep{platt1999probabilistic, zadrozny2001obtaining, niculescu2005predicting, naeini2015}.
Calibration has been the focus of renewed attention after \citet{guo2017} observed that deep neural networks of their time exhibited severe miscalibration, often in the direction of over-confidence.
They also noted that logloss tends to overfit earlier than accuracy during training.
Later studies \citep{minderer2021revisiting, tao2024a} emphasized that model calibration can vary significantly with architectural and training choices.
In particular, modern architectures like transformers \citep{vaswani2017attention} and single-pass pre-training on large datasets often produce well-calibrated models.
To address miscalibration, a wide range of post-hoc methods have been developed.
These include variants of logistic regression, such as Platt scaling \citep{platt1999probabilistic}, Temperature/Vector/Matrix scaling \citep{guo2017}, Beta calibration \citep{kull2017beta}, and Dirichlet calibration \citep{kull2019}.
Another family of methods stems from isotonic regression \citep{zadrozny2002transforming, vovk2015large, oron2017centered, berta2024classifier}.
Binning-based techniques also exist \citep[e.g.,][]{naeini2015}, though these are more commonly used for estimating calibration error than for recalibration itself.
\citet{wang2021} demonstrate that changing the training loss function can improve calibration before, but not after, post-hoc calibration.

\paragraph{Post-hoc calibration and proper scores}
Several recent works relate post-hoc calibration to proper scoring rules.
\citet{gruber2022better} evaluate post-hoc methods using a quantity akin to loss, and \citet{gruber2024optimizing} propose a variational formulation for optimizing squared calibration error estimators.
\citet{dimitriadis2021stable} evaluate miscalibration and discrimination for binary classifiers using the Brier score after isotonic regression.
\citet{ranjan2024post} propose using the loss after post-hoc transformations (e.g., temperature scaling) for selecting the best epoch, showing that the validation loss after transformations behaves differently to the validation loss, but without making the link with refinement error or proper scores. They show that this early stopping procedure yields improved model performance for computer vision classifiers and LLM instruction tuning.
\citet{ferrer2024evaluating} propose assessing calibration quality via the loss reduction achieved through recalibration.
\citet{blasiok2024does} show that a binary classifier is perfectly calibrated if no post-hoc technique can reduce its risk measured with a proper loss, without making the link to the refinement error or extending this result to the multi-class case.
The importance of optimizing the total loss and not only calibration error is emphasized, for example, by \citet{perez2023beyond}, \citet{chidambaram2024reassessing}, and \citet{ferrer2024evaluating}.
While our method employs a refinement estimator, \citet{perez2023beyond} propose a grouping loss estimator that could be applicable in our case, given that the irreducible loss is independent of~$f$. However, grouping loss is more challenging to estimate than refinement loss, since the former can also capture the irreducible risk $\inf_f \Risk(f)$.

\paragraph{High dimensional asymptotics of classification} Our asymptotic analysis of classification and refinement errors for regularized logistic regression in high dimensions builds on a well-established line of work.
\citet{dobriban2018, MaiHighDimLR} study classification error under the same data model and provide asymptotic results that are directly relevant to our setting. Related scenarios have also been analyzed by \citet{couillet2018classification, salehi2019impact}.
\citet{bach2024high} provides a novel perspective on these results through the lens of self-induced regularization, which may help interpret our expressions for calibration and refinement errors.

\section{Calibration-Refinement decomposition} 
\label{section:LossDecomposition}

We begin with an overview of the classical decompositions of proper losses into calibration and refinement errors, or calibration, sharpness and uncertainty. We then present variational formulations of these decompositions that allow new  interpretations of the various terms. Notably, the calibration error measures how much the loss can be decreased by optimal post-processing and the refinement error measures the remaining loss after this optimal post-processing.

\textbf{Notation.} Consider classification with $k$ classes. Let $\Delta_k$ denote the probability simplex $\{ p \in [0,1]^k \mid \sum_{i=1}^k p_i = 1 \}$ and $\mathcal{Y}_k$ the space of one-hot encoded labels $y \in \{0,1\}^k$ with $y_i = 1$ if the true class is $i$ and $y_i = 0$ otherwise. Assume we have a random variable $X \in \mathcal{X}$ (the feature vector), such that there is an unknown joint probability distribution $\mathcal{D}$ on $(X,Y)$. We make probabilistic predictions $p \in \Delta_k$ for the value of the true label $Y$ using a statistical model $f:\mathcal{X} \rightarrow \Delta_k$.

\subsection{Proper losses}
We evaluate predictions with a loss function $\ell:\Delta_k \times \mathcal{Y}_k \rightarrow \mathbb{R}$.
$\ell(p,y)$ assesses the quality of prediction $p$ for label $y$.
Well-known examples include the Brier score \citep{brier1950verification}, $\ell(p,y) = \|y-p\|_2^2$, and the cross-entropy (logloss) \citep{shannon1948mathematical}, $\ell(p,y) = -\sum_{i=1}^k y_i\log(p_i)$.

If the label $Y$ follows a categorical distribution $q \in \Delta_k$, we denote by $\ell(p,q)$ the expected loss $\mathbb{E}_{Y \sim q}[\ell(p,Y)]$ obtained with forecast~$p$.
A natural requirement for our loss function $\ell$ is that it should be minimized when the predicted distribution $p$ matches the true target distribution $q$. Losses that satisfy this property are called \textit{proper}.
Formally, we follow \citet{gneiting2007strictly} but reverse the sign such that proper losses should be minimized instead of maximized.
\begin{definition}[Proper loss function]
    We consider losses $\ell: \Delta_k \times \calY_k \to \bbR \cup \{\infty\}$ such that for any $p \in \Delta_k$, $\ell(p, \cdot)$ is measurable and quasi-integrable with respect to any $q \in \Delta_k$. This allows us to define $\ell(p, q) \coloneqq \mathbb{E}_{Y\sim q}[\ell(p,Y)]$, taking values in $\bbR \cup \{\infty\}$. If $\ell(q, q) \leq \ell(p,q)$ for all $p$ and $q$, $\ell$ is said to be \emph{proper}.
    If the inequality holds with equality only for $p = q$, $\ell$ is \emph{strictly proper}.
\end{definition}

The integrability conditions could be relaxed to use convex subsets of $\Delta_k$, which we omit for simplicity.
Leaving aside the technical details, the important takeaway is that loss functions generally used in machine learning such as the Brier score and logloss are strictly proper. 

\subsection{Decompositions}
Let $d_\ell(p,q) \coloneqq \ell(p,q) - \ell(q,q)$ and $e_\ell(q) \coloneqq \ell(q,q)$ denote the divergence and entropy associated with loss function~$\ell$ (we refer to these as $\ell$-divergence and $\ell$-entropy). For different choices of $\ell$, we recover well-known entropy and divergence functions.

\ifarxiv
\else
  \vspace{\baselineskip}
\fi
\begin{center}
\begin{tabular}{ccc}
 \toprule
 Loss $\ell$ & Divergence $d_\ell$ & Entropy $e_\ell$  \\
 \midrule
 Logloss  & KL divergence & Shannon entropy \\
 $-\sum_i y_i \log(p_i)$  & $\sum_i q_i \log\frac{q_i}{p_i}$ & $-\sum_i q_i \log q_i$ \\ 
 \midrule
 Brier score & Squared distance & Gini index \\
 $\|y-p\|_2^2$ & $\|p-q\|_2^2$ & $\sum_i q_i(1-q_i)$ \\ 
 \bottomrule
\end{tabular}
\end{center}
\ifarxiv
\else
  \vspace{\baselineskip}
\fi

Given a probabilistic classifier $f$, let $C \coloneqq \mathbb{E}_\mathcal{D}[Y|f(X)]$ denote the conditional expectation of $Y$ given the prediction $f(X)$.  We refer to values of $C$ as \emph{calibrated scores}.
The risk of $f$ measured with a proper loss $\ell$ satisfies the ``calibration-refinement decomposition'' \citep{brocker2009},
\begin{equation}\label{equation:DecompositionBrocker}
    \Risk(f) = \bbE[\ell(f(X), Y)] = \underbrace{\bbE[d_\ell(f(X), C)]}_{\calK_\ell(f)} + \underbrace{\mathbb{E}[e_\ell(C)]}_{\calR_\ell(f)} \, .
\end{equation}

The first term in decomposition (\ref{equation:DecompositionBrocker}) is the \textit{calibration error} $\calK_\ell(f)$ associated with $\ell$. The second term, denoted $\calR_\ell(f)$, is the \textit{refinement error} for $\ell$.
The refinement term can be further decomposed into the \textit{uncertainty} of $Y$ $\calU_\ell(Y)$ (notice that it is independent of $f$), and a \textit{sharpness} term $\calS_\ell(f)$, yielding the ``calibration-sharpness decomposition'' \citep{brocker2009},
\begin{equation}\label{equation:DecompositionBrockerSharpness}
    \Risk(f) = \bbE[\ell(f(X), Y)] = \underbrace{\bbE[d_\ell(f(X), C)]}_{\calK_\ell(f)} + \underbrace{ \underbrace{e_\ell(\bbE[Y])}_{\calU_\ell(Y)}  -\underbrace{\bbE[d_\ell(\bbE[Y], C)]}_{\calS_\ell(f)}}_{\calR_\ell(f)} \, .
\end{equation}

Alternatively, refinement error can be decomposed into the irreducible loss (risk of the optimal classifier $f^*$), and the grouping loss \citep{brocker2009, kull2015novel, perez2023beyond}; see \Cref{sec:appendix:calref} for more details.

\subsection{Calibration error}
Calibration has been the subject of a significant body of research on its own, not only in the context of the calibration-refinement decomposition. A classifier $f$ is said to be \textit{calibrated} if $f(X) = C$ almost surely. When this is satisfied, for a given prediction $f(X) = p$, the expected outcome $\mathbb{E}_\mathcal{D}[Y|f(X)=p]$ is aligned with $p$. This is a desirable property in that it yields model predictions that are interpretable as the probabilities that $Y$ belongs to each of the $k$ classes.

\begin{remark}
Calibration is traditionally defined as $\mathbb{P}_\mathcal{D}(Y|f(X)) = f(X)$ almost surely.
Notice that since labels $Y \in \{0,1\}^k$ are one-hot encoded, the conditional expectation $C=\mathbb{E}_\mathcal{D}[Y|f(X)]$ matches the conditional class probabilities $\mathbb{P}_\mathcal{D}(Y|f(X))$.
\end{remark}

\paragraph{Calibration error} Model miscalibration is measured by the expected gap between $f(X)$ and $C$. The choice of a distance or divergence function $d$ gives rise to different notions of \emph{calibration error}: $\calK^{(d)}(f) = \bbE[d(f(X), C)]$. A popular choice for $d$ is the $L^1$ distance, resulting in the expected calibration error \citep[ECE,][]{naeini2015}. Note that when choosing the $L^2$ distance or KL divergence, we recover ``proper'' calibration errors \citep{gruber2022better} that appear in the decomposition of the Brier score and the logloss.
In practice, the distribution $\mathcal{D}$ is only known via a finite set of samples, $(x_i, y_i)_{1\leq i \leq n}$, and $f$ makes continuous predictions $f(x_i) \in \Delta_k$.
While some calibration metrics exhibit good statistical properties \citep{imanol2022metrics}, estimating the calibration error itself is difficult.
$C$ is often approximated by binning predictions on the simplex, resulting in estimators that are biased and inconsistent \citep{kumar2019, vaicenavicius2019, roelofs2022}.
In the multi-class case, the curse of dimensionality makes estimation even more difficult.
Weaker notions such as class-wise or top-label calibration error are often used \citep{kumar2019, kull2019}.

\paragraph{Post-hoc calibration} Many machine learning classifiers suffer from calibration issues, which has given rise to a family of techniques known as ``post-hoc calibration.'' These methods reduce calibration error after training using a reserved set of samples $\mathcal{C}$ called a ``calibration set.'' Well-known examples include isotonic regression and temperature scaling.

\textbf{Isotonic regression \citep[IR,][]{zadrozny2002transforming}} finds the monotonic function $g^*$ that minimizes the risk of $g \circ f$ on the calibration set. It is a popular calibration technique for binary classifiers. The multi-class extensions come with important drawbacks, however.

\textbf{Temperature scaling \citep[TS,][]{guo2017}} optimizes a scalar parameter $\beta$ to rescale the log-probabilities. Formally, it learns the function $g_{\beta^*}$ on the calibration set with
\begin{IEEEeqnarray*}{+rCl+x*}
g_\beta(p) & \coloneqq & \softmax(\beta\log(p)) \, , \\
\beta^* & \coloneqq & \argmin_{\beta \in \bbR_+} \calL(\beta) \, , \\
\calL(\beta) & \coloneqq & \sum_{(x, y) \in \mathcal{C}} \ell(g_\beta(f(x)), y) \, . \IEEEyesnumber\label{eq:temperatureScaling}
\end{IEEEeqnarray*}
For neural networks, this comes down to rescaling the last layer by $\beta^*$. TS has many advantages: it is efficient, applies to multi-class problems, does not affect accuracy, and is robust to overfitting since it only has one parameter.

We refer the interested reader to \citet{silva2023} for a detailed review of calibration.

\subsection{Refinement error}
Given the appeal of calibrated predictions, significant effort has been devoted to estimating and reducing calibration error. The fact that it interacts with another term to form the overall risk of the classifier, however, has received much less attention. This is partly due to the fact that refinement error is not well known. So what exactly is refinement?

In decomposition (\ref{equation:DecompositionBrocker}), refinement appears as the expected $\ell$-entropy of $Y$ given $f(X)$: $\mathcal{R}_\ell(f) = \mathbb{E}[e_\ell(\mathbb{E}[Y|f(X)])]$. We see that if $Y$ is fully determined by $f(X)$, $Y|f(X)$ has no entropy and the refinement error is null. In contrast, if $Y|f(X)$ is as random as $Y$, it is maximal. Refinement error thus quantifies how much of the variability in $Y$ is captured by $f(X)$, with a lower error indicating greater information. Just like classification error, it assesses the ability to distinguish between classes independently of calibration issues such as over/under-confidence. It provides a much more comprehensive measure, however, as it considers the entire distribution of $f(X)$ rather than a single discretization. This comes at a cost; in particular, as with the calibration error, it is much harder to estimate.

\begin{remark}
    As we have seen in \Cref{equation:DecompositionBrockerSharpness}, refinement can be further decomposed into the uncertainty of $Y$ $\calU_\ell(Y)$ minus a sharpness term $\calS_\ell(f) = \bbE[d_\ell(\bbE[Y], C)]$, which quantifies the expected gap between the calibrated scores and the best constant predictor $\mathbb{E}[Y]$.
    Intuitively, we see that the more information $f(X)$ contains about $Y$, the more $C$ will deviate from $\mathbb{E}[Y]$, the larger the sharpness. Sharpness and refinement both measure the discrimination power of the classifier $f$; note, however, that larger sharpness corresponds to smaller refinement.
    For a fixed data distribution $\mathcal{D}$, the uncertainty $\calU_\ell(Y)$ is constant, so sharpness and refinement are the same up to a constant term.
    In this paper, we choose to consider refinement error, to avoid the need to include an additional constant term in the loss decomposition, but our methods apply directly to sharpness.
\end{remark}

\subsection{A variational formulation}

\begin{figure*}
    \centering
    \resizebox{\textwidth}{!}{
    \begin{tikzpicture}
        \newcommand{\xscale}{3.5cm}
        \newcommand{\xbayes}{\xscale}
        \newcommand{\xref}{2*\xscale}
        \newcommand{\xf}{3*\xscale}
        \newcommand{\xconst}{4*\xscale}
        \newcommand{\xend}{4.5*\xscale}
        \newcommand{\ticky}{0.1cm}
        \newcommand{\labely}{-0.5cm}
        \newcommand{\biglabely}{-0.8cm}
        
        \draw[->] (0, 0) -- (\xend, 0);
        \draw (0, -\ticky) -- (0, \ticky);
        \draw (\xbayes, -\ticky) -- (\xbayes, \ticky);
        \draw (\xref, -\ticky) -- (\xref, \ticky);
        \draw (\xf, -\ticky) -- (\xf, \ticky);
        \draw (\xconst, -\ticky) -- (\xconst, \ticky);

        \node (a) at (0, \labely) {$0$};
        \node (a) at (\xbayes, \biglabely) {$$
          \shortstack{
            $\displaystyle \Risk(f^*)$\\
            $\displaystyle = \min_f \Risk(f)$
          }
        $$};
        \node (a) at (\xref, \biglabely) {$$
          \shortstack{
            $\displaystyle \Risk(g^* \circ f)$\\
            $\displaystyle = \min_{g} \Risk(g \circ f)$
          }
        $$};
        \node (a) at (\xf, \labely) {$\Risk(f)$};
        \node (a) at (\xconst, \biglabely) {$$
          \shortstack{
            $\displaystyle \Risk(c^*)$\\
            $\displaystyle = \min_{c\text{ constant}} \Risk(c)$
          }
        $$};

        \draw[-] (\xf, 0) to[out=165, in=15] node[anchor=south, yshift=-0.05cm]{cal.\ err.} (\xref, 0);
        \draw[-] (\xref, 0) to[out=165, in=15] node[anchor=south, yshift=-0.05cm]{grp.\ loss} (\xbayes, 0);
        \draw[-] (\xbayes, 0) to[out=165, in=15] node[anchor=south, yshift=-0.05cm]{irred.\ loss} (0, 0);
        \draw[-] (\xref, 0) to[out=140, in=40] node[anchor=south]{refinement error} (0, 0);
        \draw[-] (\xf, 0) to[out=140, in=40] node[anchor=south]{reducible loss} (\xbayes, 0);
        \draw[-] (\xconst, 0) to[out=140, in=40] node[anchor=south]{sharpness} (\xref, 0);
    \end{tikzpicture}}
    \vspace{-2em}
    \caption{Relation between different terms relating to proper scoring rules, proven in \Cref{thm:variational_extended}. Here, infima are over measurable functions. A line $A \stackrel{u}{\rule[0.5ex]{2em}{0.4pt}} B$ means $u = B-A$. ``irred.'' stands for irreducible and ``grp.'' for grouping. For the definitions, see \Cref{rem:kullflach}.}
    \label{fig:extended_diagram}
\end{figure*}

We introduce novel variational formulations of the terms in decompositions (\ref{equation:DecompositionBrocker}) and (\ref{equation:DecompositionBrockerSharpness}).
\Cref{fig:extended_diagram} illustrates these decompositions under this novel variational perspective (proofs and additional details are deferred to \Cref{sec:appendix:calref}).

\begin{restatable}{theorem}{thmVariational}
\label{thm:variational_extended}
    Let $k \geq 2$, let $\ell: \Delta_k \times \calY_k \to \bbR \cup \{\infty\}$ be a proper loss, and let $f: \calX \to \Delta_k$ be measurable, where $\Delta_k$ is equipped with the Borel $\sigma$-algebra. Then,
    \begin{IEEEeqnarray*}{+rCl+x*}
        \calR_\ell(f) & = & \min_g \Risk(g \circ f)~, \\
        \calK_\ell(f) & = & \Risk(f) - \min_g \Risk(g \circ f)~, \\
        \calS_\ell(f) & = & \min_{c\text{ constant}} \Risk(c) - \min_g \Risk(g \circ f)~, \\
        \calU_\ell(Y) & = & \min_{c\text{ constant}} \Risk(c)~.
    \end{IEEEeqnarray*}
    Here, the minima are over measurable functions $g: \Delta_k \to \Delta_k$ and $c: \calX \to \Delta_k$. They are attained by $g^*(q) \coloneqq \bbE[Y|f(X)=q]$ and $c^*(x) \coloneqq \bbE[Y]$. If $\ell$ is strictly proper, the minimizers are unique up to $P_{f(X)}$-null sets.
\end{restatable}

\citet{blasiok2024does} show that a binary classifier is perfectly calibrated if its risk, measured by a proper loss, cannot be improved by post-processing $g \circ f$. \Cref{thm:variational_extended} extends this idea to the general setting and makes it quantitative by showing the connection to the refinement and calibration errors.
This variational formulation provides a new perspective on post-hoc calibration, showing that minimizing the risk of $g \circ f$ is a principled way to remove the calibration error of $f$.
Indeed, we see that calibration error evaluates the difference between the original risk and the risk after optimal recalibration.
This optimal relabeling is given by the calibrated score $C=\mathbb{E}[Y|f(X)]$.
Refinement error thus measures the risk that remains after optimal relabeling, in other words the risk of the calibrated scores.
This suggests that the refinement error can be estimated using the risk after post-hoc calibration while calibration error can be estimated by the difference between risk before and risk after recalibration.
Note that this yields simple estimators even in the multi-class case, which, to the best of our knowledge, is a problem with no satisfactory answer in the existing literature.

\section{Our method: refinement-based stopping} 
\label{section:OurMethod}

\begin{figure}[h]
\begin{center}
\centerline{\includegraphics[width=0.425\textwidth]{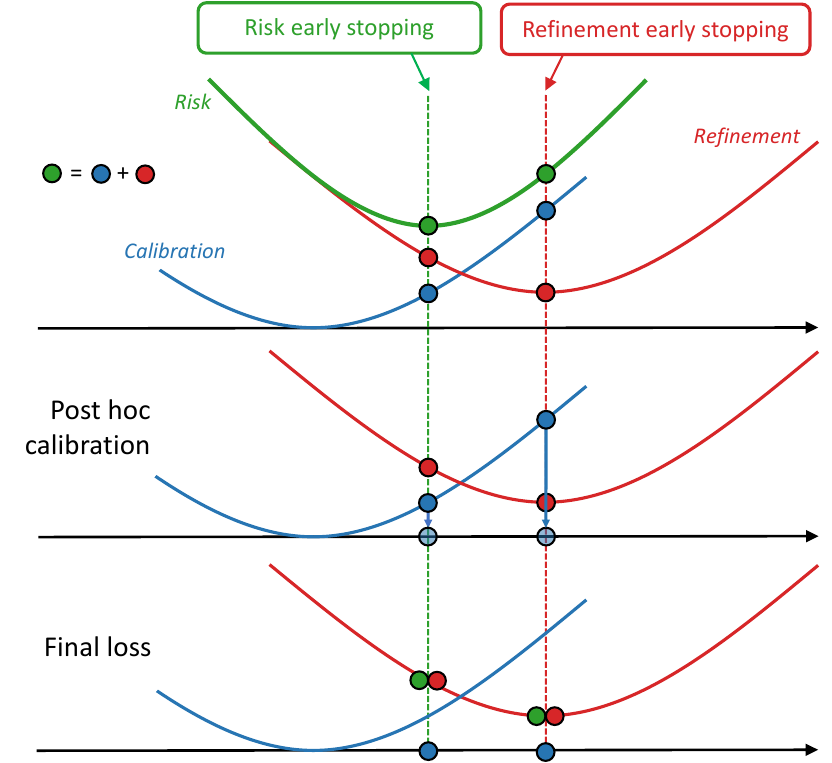}}
\vspace{-0.3em}
\caption{Risk-based versus refinement-based early stopping, before and after post-hoc calibration. Dots denote risk (green), refinement (red) and calibration (blue) errors.}
\label{figure:EarlyStopping}
\end{center}
\end{figure}

\subsection{Refinement-based early stopping: the intuition}

\paragraph{Two distinct minimizers} Minimizing the risk of a classifier comes down to minimizing simultaneously refinement and calibration errors. Empirically, these two errors are not minimized simultaneously during training, as in \Cref{figure:TrainingCurve}. One possible scenario, archetypal in deep learning, is that the training set becomes well separated, forcing the model to make very confident predictions to keep training calibration error small.
In case of a train-test generalization gap, the model becomes over-confident on test data \citep{carrell2022calibration}, while refinement error could still decrease.
Empirically, as we will see, we observe separate minimizers for the two errors across various machine learning models, whether through iterative training or when tuning a regularization parameter.
In \Cref{section:HighDimensionalLogisticRegression}, we analyze this phenomenon theoretically for high-dimensional logistic regression, showing that it arises even in very simple settings.
The consequence is that risk-based early stopping finds a compromise point between two conflicting objectives: calibration error minimization and refinement error minimization. This corroborates the widely recognized observation that many models are poorly calibrated after training \citep{guo2017}.
It also aligns with recent findings that large models trained with a single data pass are approximately calibrated \citep{achiam2023gpt}, as they are less exposed to the issue of overfitting the training set.
Fine-tuning, however, often reintroduces the multi-epoch training paradigm and with it the risk of degrading model calibration.

\paragraph{Refine, then calibrate} The question that now arises is: how do we find a model that is optimal for both calibration and refinement? In general, strictly increasing post-hoc calibration techniques such as TS leave refinement unchanged. \citet{berta2024classifier} showed that binary isotonic regression preserves the ROC convex hull, a proxy for refinement. Given that calibration error is minimized after training, with no impact on refinement, we argue that training should be dedicated to the latter only. We propose selecting the best epoch based on validation refinement error instead of validation loss as a new training paradigm for machine learning classifiers. As illustrated in \Cref{figure:EarlyStopping}, this procedure leads to smaller loss after post-hoc calibration. Compared with the standard paradigm that minimizes both refinement and calibration during training, we propose \textit{refining} during training, then \textit{calibrating} post hoc.

\ifarxiv
\else
  \vspace{\baselineskip}
\fi
\begin{center}
\begin{tabular}{c | c c} 
 \toprule
 \thead{Early stopping} & \thead{Training minimizes}  & \thead{Post hoc minimizes} \\
 \midrule
 \thead{Risk} & \thead{Cal. + Ref.} & \thead{Cal.} \\
 \midrule
 \thead{Refinement} & \thead{Ref.} & \thead{Cal.} \\
 \bottomrule
\end{tabular}
\end{center}
\ifarxiv
\else
  \vspace{\baselineskip}
\fi

\subsection{Refinement-based stopping made practical}

\paragraph{Refinement estimation} As discussed earlier, calibration error estimators are biased, inconsistent, and break in the multi-class case. Since refinement is risk minus calibration, these issues carry over to refinement estimation. One could circumvent this by using proxies like validation accuracy. \Cref{thm:variational_extended} provides a better alternative: refinement error equals risk after optimal relabeling $\min_g \text{Risk}(g \circ f)$.
Since the validation set $(x_i, y_i)_{1\leq i \leq n}$ is finite, minimizing over all measurable functions $g:\Delta_k \rightarrow\Delta_k$ would lead to drastic overfitting.
However, given a class of functions $\mathcal{G}$ that we intend to use for post-hoc calibration, we can estimate refinement error by solving
\begin{equation}\label{equation:RefinementEstimator}
    \hat{\mathcal{R}_\ell}(f) = \min_{g \color{red}\in \color{red}\mathcal{G}} \frac{1}{n} \sum_{i=1}^n\ell(g(f(x_i)), y_i)
\end{equation}
after every epoch. This comes down to selecting the best training epoch based on ``validation loss after post-hoc calibration.'' Although this is a biased estimator of refinement error, the fact that it can improve performance accords with a recent study by \citet{ranjan2024post}, who report performance gains when selecting the best training epoch based on losses after different post-hoc transformations.
Our observation that calibration and refinement errors are not minimized simultaneously, and our variational  formulation of refinement, provide a theoretical grounding for understanding this kind of empirical observation.
Returning to \Cref{equation:RefinementEstimator}, note that when compared with \Cref{thm:variational_extended}, in \Cref{equation:RefinementEstimator} we limit ourselves to functions in class $\mathcal{G}$. A trade-off arises vis-\`a-vis the size of $\mathcal{G}$. Larger function classes can detect more complex miscalibration patterns but are more prone to overfitting the validation set, resulting in poor estimation of the true refinement error. 
For our method to work, we need to choose a small class $\mathcal{G}$ that comes as close as possible to the optimal re-mapping $g^*(f) = \mathbb{E}[Y|f(X)]$.

\paragraph{TS-refinement for neural nets} We let ``TS-refinement'' refer to the estimator obtained when $\mathcal{G}$ is the class of functions generated by temperature scaling. It is parametrized by a single scalar parameter, limiting the danger of overfitting the validation set. Moreover, TS is very effective at reducing the calibration error of neural networks (however large it is) while leaving refinement unchanged. This is supported by a plethora of empirical evidence \citep{guo2017, wang2021}. Finally, we demonstrate in \Cref{section:GaussianDataModel} that under assumptions on the logit distribution, the optimal re-mapping $g^*$ is attained by TS, making TS-refinement a consistent estimator.

In practice, when training a neural network, we recommend selecting the best epoch in terms of validation loss after temperature scaling. Going back to the example in \Cref{figure:TrainingCurve}, this procedure would select a later epoch (red dot), for which validation loss (green curve) is larger but validation loss after TS (red curve) is smaller. Given that we envisage using TS for post-hoc calibration anyway, this simple procedure guarantees a smaller validation risk.

\begin{figure}[h]
    \centering
    \includegraphics[width=0.5\textwidth]{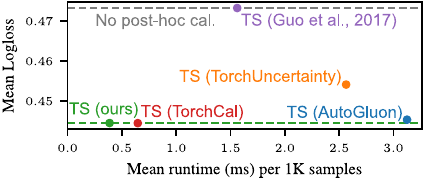}
    \vspace{-0.5em}
    \caption{\emph{Runtime versus mean benchmark scores of different TS implementations.} Runtimes are averaged over validation sets with 10K+ samples. Evaluation is on XGBoost models trained with default parameters, using the epoch with the best validation accuracy.}
    \label{fig:calib-benchmark}
\end{figure}

\paragraph{Improving TS implementations} Using TS-refinement for early stopping requires a fast and robust temperature-scaling implementation. Surveying existing codebases made it clear that there was room for improvement in terms of speed and even performance.
For the logloss, the objective $\calL(\beta)$ in Eq.~\eqref{eq:temperatureScaling} is convex. Most existing implementations apply L-BFGS \citep{liu1989limited} to optimize $\calL(1/T)$, which is unimodal but nonconvex. Additionally, the step-size choice for L-BFGS sometimes led to suboptimal performance. Instead, since $\calL'$ is increasing by convexity of $\calL$, we propose using bisection to find a zero of $\calL'$. \Cref{fig:calib-benchmark} shows that on the benchmark from \Cref{section:Experiments}, our implementation achieves equal or lower test loss and is faster than alternatives from \citet{guo2017}, TorchUncertainty, TorchCal \citep{ranjan2023torchcal,feinman2021pytorch}, and AutoGluon \citep{erickson2020autogluon}. We provide details in \Cref{appendix:TemperatureScaling}.

\paragraph{Other estimators} Our method applies to a wide range of machine learning models. It can help select the best step in any iterative training procedure (e.g., boosting) or the level of regularization for non-iterative models. Refinement does not have to be estimated with TS-refinement; any proxy available can be selected. Outside of the neural network setting, TS may be less effective. In the binary setting, using risk after isotonic regression (IR-refinement) might be a good option. It is consistent for a larger class of functions than TS-refinement since the class $\mathcal{G}$ contains all monotonic functions. The risk of overfitting becomes larger, however, and must be considered carefully.

\begin{remark}
    The proper loss that we minimize with early stopping can be different from the one used for training (which is often logloss).
    In \Cref{sec:appendix:tabular} we experiment with early stopping on Brier-refinement instead of logloss-refinement. Any proper loss can be targeted by refinement early stopping, even non-strict proper losses like classification error.
\end{remark}

\section{Experiments}
\label{section:Experiments}

\paragraph{Computer Vision} We benchmark our method by training a ResNet-18 \citep{he2016} and WideResNet (WRN) \citep{zagoruyko2016} on CIFAR-10, CIFAR-100 and SVHN datasets \citep{krizhevsky2009, netzer2011}. We reserve 10\% of the training set for validation. We train for 300 epochs using SGD with momentum, weight decay, and a step learning rate scheduler. We use random cropping, horizontal flips, and cutout \citep{devries2017}. For reference, this training procedure allows us to reach $95\%$ accuracy on CIFAR-10 with the ResNet-18 and over $96\%$ with the WRN. We train the models ten times on each dataset. The code to run the benchmark is available online\footnote{\url{github.com/eugeneberta/RefineThenCalibrate-Vision}}. On NVIDIA V100 GPUs, the benchmark took around 300 GPU hours to run.
In \Cref{figure:ComputerVisionBenchmark} we report the relative differences in test logloss obtained with different training procedures. After early stopping on validation loss (first column), using TS results in significant improvement (second column). Early stopping on TS-refinement instead of validation loss leads to even better results (third column). We also compare with validation-accuracy-based early stopping (fourth column), another refinement proxy, showing that TS-refinement more consistently yields smaller loss.

We observed empirically that the learning rate scheduler and regularization strength have an influence on the behavior of calibration error during training. This has consequences for the effect observed with our method. In general though, validation TS-refinement remains the best estimator available for test loss after temperature scaling. Brier score is less sensitive than logloss to sharp variations in calibration error and can be a good alternative for early stopping for practitioners who want to avoid fitting TS every epoch. We defer details on our training procedure and complete results, including additional metrics like accuracy, ECE and smooth ECE \citep{blasiok2024smooth} to \Cref{appendix:ComputerVision}.

\begin{figure*}[h]
\begin{center}
\centerline{\includegraphics[width=\textwidth]{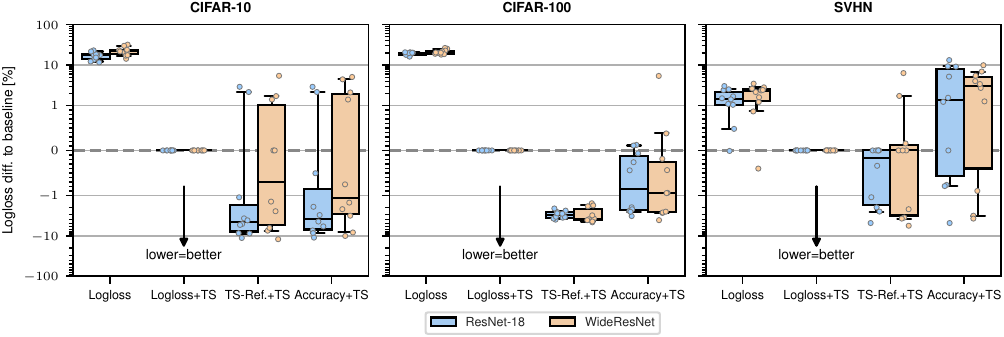}}
\vspace{-0.5em}
\caption{\emph{Relative differences in test logloss (lower is better) between logloss+TS and other procedures on vision datasets.} ``+TS'' indicates temperature scaling applied to the final model. Each dot represents a training run on one dataset. Box-plots show the 10\%, 25\%, 50\%, 75\%, and 90\% quantiles. Relative differences (y-axis) are plotted using a log scale.}
\label{figure:ComputerVisionBenchmark}
\end{center}
\end{figure*}

\textbf{Tabular data}, where $X$ is a vector of heterogeneous numerical and/or categorical features, is ubiquitous in ML applications. We take 196 binary and multi-class classification datasets from the benchmark from \citet{ye2024closer}, containing between 1K and 100K samples after subsampling the largest datasets; see \Cref{sec:appendix:tabular}. We evaluate three methods:
\begin{itemize}
    \item \textbf{XGBoost} \citep{chen2016xgboost} is a popular implementation of gradient-boosted decision trees, with strong performance on tabular benchmarks \citep{grinsztajn2022tree}. Due to its iterative optimization, early stopping is relevant for XGBoost as well.
    \item \textbf{MLP} is a simple multilayer perceptron, similar to the popular MLP baseline by \citet{gorishniy2021revisiting}.
    \item \textbf{RealMLP} \citep{holzmuller2024better} is a recent state-of-the-art deep learning model for tabular data \citep{ye2024closer, erickson2025tabarena}, improving the standard MLP in many aspects.
\end{itemize}

For each dataset, we report the mean loss over five random splits into 64\% training, 16\% validation, and 20\% test data. For each split, we run 30 random hyperparameter configurations and select the one with the best validation score. The validation set is also used for stopping and post-hoc calibration. Computations took around 40 hours on a 32-core CPU (for XGBoost) and four RTX 3090 GPUs (for NNs).

\begin{figure*}
    \centering
    \begin{minipage}{0.48\textwidth}
    \centering
    \includegraphics[width=\textwidth]{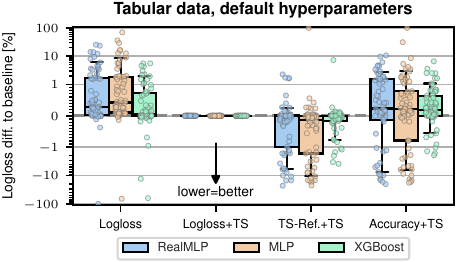}
    \end{minipage}
    \begin{minipage}{0.48\textwidth}
    \centering
    \includegraphics[width=\textwidth]{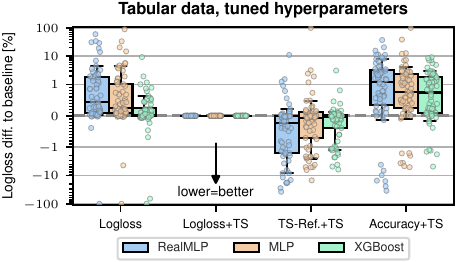}
    \end{minipage}
    \vspace{-0.5em}
    \caption{\emph{Relative differences in test logloss (lower is better) between logloss+TS and other procedures on tabular datasets.} ``+TS'' indicates temperature scaling applied to the final model. Each dot represents one dataset from \citet{ye2024closer}, using the 65 datasets with 10K+ samples. Percentages are clipped to $[-100, 100]$ due to one outlier with almost zero loss. Box-plots show the 10\%, 25\%, 50\%, 75\%, and 90\% quantiles. Relative differences (y-axis) are plotted using a log scale.}
    \label{fig:tabular}
\end{figure*}

\Cref{fig:tabular} shows the results on datasets with at least 10K samples. Generally, we observe that TS as well as stopping/tuning on TS-refinement helps on most datasets, while stopping on accuracy often yields poor logloss even after TS. Our extended analysis in \Cref{sec:appendix:tabular} shows that results are noisier and less clear for small or easy datasets. However, stopping on TS-refinement frequently yields better accuracy and AUC than stopping on logloss, and hence strikes an excellent balance between different metrics. %

\section{Decomposition in the Gaussian data model} 
\label{section:GaussianDataModel}

In this section and the next, we embark on a theoretical analysis that provides insight into the calibration-refinement decomposition. We demonstrate that even with a simple data model and predictor, calibration and refinement errors are not minimized simultaneously. This highlights the relevance of refinement-based early stopping and underscores the broad impact of the problem we address.

We use the following stylized model. Consider the feature and label random variables $X \in \mathbb{R}^p, Y \in \{-1,1\}$ under the two-class Gaussian model $X \sim \mathcal{N}(\mu, \Sigma)$ if $Y=1$ and $X \sim \mathcal{N}(-\mu, \Sigma)$ if $Y=-1$ for some mean $\mu \in \mathbb{R}^p$ and covariance $\Sigma \in \mathbb{R}^{p \times p}$. We further assume balanced classes, $\mathbb{P}(Y=1) = \mathbb{P}(Y=-1) = \frac{1}{2}$. Under this data model, $\mathbb{P}(Y=1|X=x) = \sigma({w^*}^\top x)$ with $w^* = 2 \Sigma^{-1} \mu$ and $\sigma(x) = \frac{1}{1+e^{-x}}$ denotes the sigmoid function. The sigmoid function's shape is well suited to describe the posterior probability of $Y$ given~$X$. Note that this holds for any pair $P(X|Y=\pm 1)$ from the same exponential family \citep{jordan1995logistic}. Proofs for this section are deferred to \Cref{appendix:ProofsGaussianDataModel}.

We are interested in the calibration and refinement errors of the linear model $f(x) = \sigma(w^\top x)$.
The weight vector $w$ can be learned with techniques such as logistic regression or linear discriminant analysis \citep{fisher1936}. Denoting $a_w = \langle w, w^*\rangle_\Sigma / \|w\|_\Sigma$ with $\langle w, w^*\rangle_\Sigma = w^\top  \Sigma w^*$ and $\|w\|_\Sigma = \sqrt{w^\top  \Sigma w}$, the error rate of the linear model can be written as $\err(w) = \Phi(-\frac{a_w}{2})$ where $\Phi(x)=\frac{1}{\sqrt{2\pi}} \int_{-\infty}^x \exp(-\frac{t^2}{2}) dt$ is the cumulative distribution function of the standard normal distribution. Notice that $a_w$ is invariant by rescaling of $w$. It measures the alignment with the best model $w^*$ independently of the weight vector's norm. We refer to $a_w$ as the ``expertise level'' of our model~$f$. Interestingly, $a_w$ also appears in the calibration and refinement errors for this simple model.
\begin{restatable}{proposition}{thmCalRefError} \label{theorem:CalRefError}
    For a proper loss $\ell$ that is symmetric, i.e., $\ell(p, 1) = \ell(1-p, 0)$ (as is the case for logloss and Brier score), the calibration and refinement errors of our model are
    \begin{align*}
    \mathcal{K}_\ell(w) &= \mathbb{E}\Big[ d_\ell \Big( \sigma\Big( \|w\|_\Sigma \Big( z + \frac{a_w}{2} \Big) \Big), \sigma\Big( a_w \Big( z + \frac{a_w}{2} \Big) \Big) \Big) \Big] \, ,\\
    \mathcal{R}_\ell(w) &= \mathbb{E}\Big[ e_\ell \Big( \sigma\Big(a_w \Big( z + \frac{a_w}{2}\Big) \Big) \Big) \Big] \, ,
    \end{align*}
    where the expectation is taken with respect to $z \sim \mathcal{N}(0,1)$.
\end{restatable}
Notice that the refinement error is a decreasing function of $a_w$, just like the error rate. Moreover, it depends on $w$ via $a_w$ only and so it is invariant by rescaling of the weight vector. In contrast, calibration error is not invariant by rescaling $w$. It is minimized, and reaches zero, when $\|w\|_\Sigma = a_w$. The larger the norm of the weight vector $\|w\|_\Sigma$, the larger the predicted probabilities $\sigma(w^\top x)$, so $\|w\|_\Sigma$ can be interpreted as the model's ``confidence level.''
The model is calibrated when this confidence level equals the expertise level $a_w$. A consequence is that there always exists a rescaling of the weight vector for which the calibration error cancels.

\begin{restatable}{proposition}{thmTemperatureScaling}\label{theorem:TemperatureScaling}
    Assume $\langle w, w^*\rangle_\Sigma > 0$. The rescaled weight vector $w_s \leftarrow s w$ with $s = \langle w, w^*\rangle_\Sigma / \|w\|_\Sigma^2 > 0$ yields null calibration error $\mathcal{K}(w_s) = 0$ while preserving the refinement error $\mathcal{R}(w_s) = \mathcal{R}(w)$.
\end{restatable}
Note that such rescaling of the weight vector exactly corresponds to temperature scaling. For classifiers such as logistic regression or neural nets that use a sigmoid to produce probabilities, \Cref{theorem:TemperatureScaling} establishes that, under the assumption that the logit distribution is a mixture of Gaussians, TS sets calibration error to zero while preserving refinement. An immediate corollary of this and \Cref{thm:variational_extended} is that refinement error equals risk after temperature scaling: $\mathcal{R}(w) = \min_{s \in \mathbb{R}} \Risk(sw)$.

\begin{remark}
    For non-centered or imbalanced data, adding an intercept parameter to TS is necessary to achieve the same result. A similar analysis can be derived for the multi-class case, leading to post-hoc calibration with matrix scaling.
\end{remark}

This supports the idea that we should minimize refinement during training. Rescaling the weight vector to the correct confidence level is enough to remove the calibration error after training. From our analysis, it seems clear that calibration and refinement error minimizers do not have to match. The confidence level $\|w\|_\Sigma$ is not constrained to equal the expertise level $a_w$ in any way. We observed empirically that loss and refinement minimizers can be separated. Can we push the analysis further to exhibit this theoretically?

\section{High dimensional asymptotics of logistic regression} 
\label{section:HighDimensionalLogisticRegression}

In this section, we place ourselves in the setting of regularized logistic regression to study the impact of calibration and refinement on the risk during training. Proofs and missing technical details are provided in \Cref{appendix:BetaSpectral}.

\begin{figure}[h]
\begin{center}
\centerline{\includegraphics[width=0.5\textwidth]{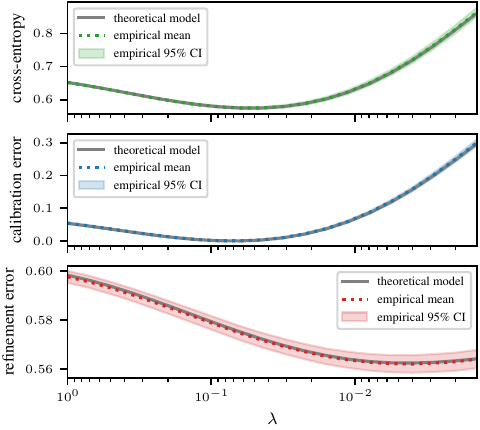}}
\vspace{-0.5em}
\caption{Cross-entropy, calibration and refinement errors when $\lambda$ varies. The spectral distribution $F$ is uniform ($\alpha = \beta = 1$), $e^*=10\%$, and $r=\frac{1}{2}$. We fit a logistic regression on 2000 random samples from our data model to obtain $\hat{w}_\lambda$. We compute the resulting calibration and refinement errors using \Cref{theorem:CalRefError} and plot $95\%$ error bars over 50 seeds. Theory closely matches empirical observations.}
\label{figure:firstCurve}
\end{center}
\end{figure}

Consider $n$ pairs $(x_i, y_i)_{1 \leq i \leq n}$ from the bimodal data model of \Cref{section:GaussianDataModel} and let $w_\lambda$ denote the weight vector learned by solving regularized logistic regression
$$
\min_{w \in \mathbb{R}^p} \frac{1}{n} \sum_{i=1}^n \log(1+ \exp(-y_i w^\top  x_i)) + \frac{\lambda}{2} \| w \|^2 \; ,
$$ for a given regularization strength $\lambda \in \mathbb{R}_+$.

We will exploit known results on $w_\lambda$ in the high-dimensional asymptotic setting where $n, p \rightarrow \infty$ with a constant ratio $p/n \rightarrow r > 0$. We assume that the class means on each dimension $(\mu_i)_{1 \leq i \leq p}$ are sampled i.i.d. from a distribution satisfying $\mathbb{E}[\mu_i^2] = \frac{c^2}{p}$, such that $\| \mu \|_2 \xrightarrow{p \rightarrow \infty} c$ for some constant $c > 0$ that controls the class separability of the problem. To simplify the formulas, we assume that $\Sigma$ is diagonal. We further assume its eigenvalues $(\sigma_i)_{1 \leq i \leq p}$ are sampled i.i.d. from a positive, bounded spectral distribution $F$. \citet{MaiHighDimLR} provide the following characterization of $w_\lambda$
$$
w_\lambda \sim \mathcal{N}\Big(\eta (\lambda I_p + \tau \Sigma)^{-1} \mu, \frac{\gamma}{n} (\lambda I_p + \tau \Sigma)^{-1} \Sigma (\lambda I_p + \tau \Sigma)^{-1} \Big) ,
$$
where $\tau, \eta, \gamma$ are the unique solutions to a non-linear system of equations, parametrized by $r, c, \lambda$ and $F$.
To compute calibration and refinement errors that appear in \Cref{theorem:CalRefError}, we are interested in $\langle w_\lambda, w^* \rangle_\Sigma$ and $\|w_\lambda\|_\Sigma$.

\begin{restatable}{proposition}{propAsymptoticDistributions}\label{proposition:AsymptoticDistributions}
For $n, p \to \infty$,
$$
    \langle w_\lambda, w^* \rangle_\Sigma \xrightarrow{P} \mathbb{E}_{\sigma \sim F}\Big[\frac{2 \eta c^2 }{\lambda + \tau \sigma}\Big] \, ,
    \quad 
    \|w_\lambda\|_\Sigma^2 \xrightarrow{P} \mathbb{E}_{\sigma \sim F}\Big[\frac{\gamma r \sigma^2 + \eta^2 c^2 \sigma}{(\lambda+\tau \sigma)^2} \Big] \, ,
$$
where the convergence is in probability.
\end{restatable}

\begin{remark}
    Results from \citet{dobriban2018} allow us to derive similar results for regularized linear discriminant analysis and ridge regression, leading to simpler formulations of refinement and calibration errors. We omit these results due to space constraints.
\end{remark}

The data separability parameter $c$ is hard to interpret; instead, we work with $e^*$, the error rate of the optimal classifier $w^*$.
\begin{restatable}{proposition}{propOptimalErrorRate}\label{proposition:OptimalErrorRate}
As $n,p \rightarrow \infty$,
$e^* \xrightarrow{a.s.} \Phi \big(- c \sqrt{\mathbb{E}_{\sigma \sim F}[\sigma^{-1}]} \big)$.
\end{restatable}
$e^*$ provides a common and interpretable scale for the problem difficulty, whatever the shape of the spectral distribution. Given a spectral distribution $F$, \Cref{proposition:OptimalErrorRate} gives a simple relation between $e^*$ and $c$. We can choose any bounded and positive distribution $F$ in our formulas to compute $\langle w_\lambda, w^* \rangle_\Sigma$ and $\|w_\lambda\|_\Sigma^2$. We pick $\sigma = \varepsilon + B(\alpha, \beta)$ where $B$ denotes a Beta distribution with shape parameters $(\alpha, \beta)$ and $\varepsilon$ is a small shift parameter to make $F$ strictly positive (we fix $\varepsilon = 10^{-3}$). This allows us to explore a variety of spectral distribution shapes with a single mathematical model, by tweaking $\alpha$ and $\beta$. Using Propositions \ref{proposition:AsymptoticDistributions} and \ref{proposition:OptimalErrorRate}, we obtain formulas for $\langle w_\lambda, w^* \rangle_\Sigma, \|w_\lambda\|_\Sigma^2$ and $e^*$, see \Cref{appendix:BetaSpectral}.

For a given regularization strength $\lambda$, dimensions-to-samples ratio $r$, optimal error rate $e^*$ and spectral distribution $F$ (controlled by $\alpha, \beta$), we can compute $\eta, \tau, \gamma$ by solving the system provided by \citet{MaiHighDimLR}. Using \Cref{proposition:AsymptoticDistributions} and \Cref{theorem:CalRefError} we obtain the calibration and refinement errors. We provide a fast implementation of this method\footnote{\url{github.com/eugeneberta/RefineThenCalibrate-Theory}}. For a fixed set of parameters $F, r, e^*$, we can plot a learning curve for logistic regression by varying the amount of regularization $\lambda$. We observe the contributions of calibration and refinement errors to the total logloss of the model, as in \Cref{figure:firstCurve}.
In this example, calibration and refinement errors are not minimized jointly, resulting in a loss minimizer that is suboptimal for both errors. Given that we can easily set the calibration error to zero (\Cref{theorem:TemperatureScaling}), our mathematical model predicts that refinement-based early stopping yields a $2\%$ decrease in downstream cross-entropy on this example.

We can explore the impact of problem parameters on this phenomenon. We investigate three spectral distribution shapes (first row). For each of these, \Cref{figure:heatmap} displays the gap between the calibration error minimizer and the refinement error minimizer (second row) and the relative cross-entropy gains ($\%$) obtained when stopping on refinement instead of loss (third row) as a function of the ratio $r=p/n$ (x-axis) and problem difficulty $e^*$ (y-axis).

\begin{figure}[ht]
\begin{center}
\centerline{\includegraphics[width=0.6\textwidth]{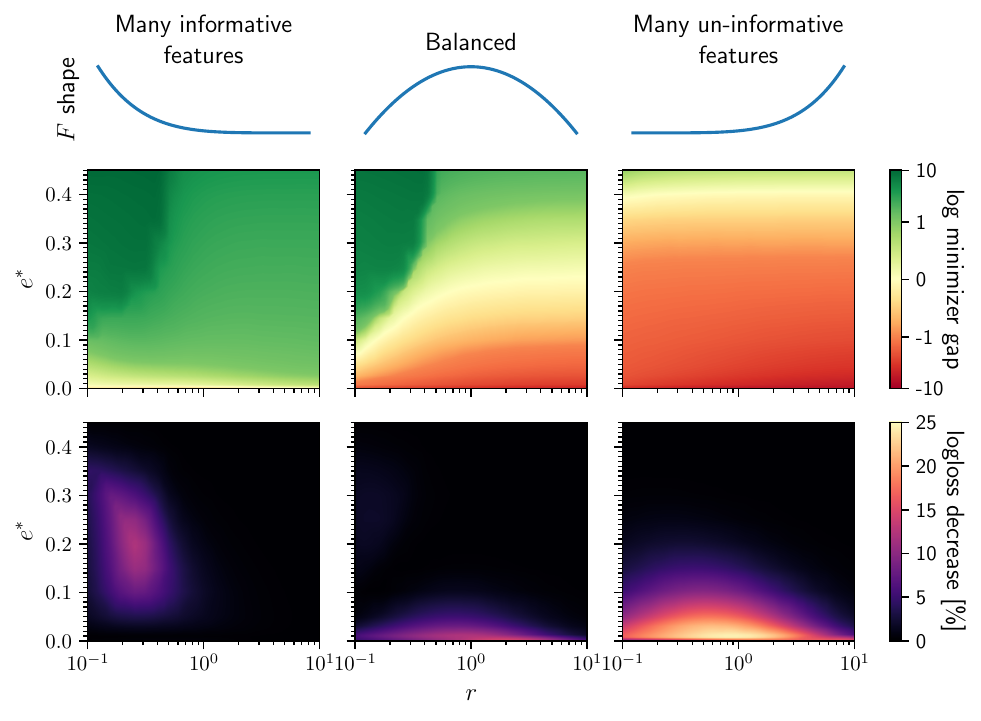}}
\vspace{-0.5em}
\caption{\emph{Influence of problem parameters on calibration and refinement minimizers.} First row: spectral distribution shape. Second row: log gap between the two minimizers. In green regions, calibration is minimized earlier, while in red regions it is refinement. Third row: relative logloss gain ($\%$) obtained with refinement-based early stopping.}
\label{figure:heatmap}
\end{center}
\end{figure}

To interpret the spectral distribution shapes, note that a small eigenvalue translates to a small ``within class variance'' and thus more separable data on the corresponding axis. Smaller eigenvalues correspond to more informative features. The left column depicts a scenario where most features are equally informative. The central column has a wide spread of good and bad features. In the right column, most features are bad while there are a few ``hidden'' good components. In the first scenario, calibration converges earlier almost everywhere on the parameter space; this corresponds to the regime we observe in computer vision. When informative features are scarce on the other hand, refinement is minimized earlier (right column), which may be an appropriate expectation for sparse problems in genomics or histopathology. When the spectral distribution is more balanced, the scenarios are co-existing, depending on problem parameters. Refinement-based early stopping always yields smaller loss, and in the regions where the two minimizers are far apart (whichever error is minimized first), the gains are significant with up to $25\%$ loss decrease.

\begin{remark}
    \citet{Bartlett2020} establish that benign overfitting occurs for linear regression in a similar data model under some condition on the spectral distribution of the covariance matrix. An avenue for future work is to study how this translates to classification and how it links with the calibration-refinement decomposition.
\end{remark}

\section{Conclusion}

We have demonstrated that selecting the best epoch and hyperparameters based on refinement error solves an issue that arises when using the validation loss. We established that refinement error can be estimated by the loss after post-hoc calibration. This allowed us to use TS-refinement for early stopping, yielding improvement in test loss on various machine learning problems. Our estimator is available in the \texttt{probmetrics} package\footnote{\url{https://github.com/probkit/probmetrics}} and can be used seamlessly with any architecture.
One limitation is that TS-refinement can be more sensitive to overfitting on small validation sets. In this case, cross-validation, ensembling, and regularized post-hoc calibrators could be useful. Moreover, we did not explore the utility of TS-refinement when training with cross-validation instead of holdout validation. Another little-explored area is the influence of training parameters—such as optimizers, schedulers, and regularization—on the calibration and refinement errors during training.
While refinement-based stopping is relevant in data-constrained settings prone to overfitting, it could also prove useful for the fine-tuning of foundation models, where training long past the loss minimum can be beneficial \citep{ouyang2022training, carlsson2024hyperfitting}.
Additionally, while we explore its utility for model optimization, our refinement estimator could be a valuable diagnostic metric, especially in imbalanced multi-class settings where accuracy and area under the ROC curve are less appropriate.

\section*{Acknowledgements}
The authors would like to thank Sacha Braun, Etienne Gauthier, Sebastian Gruber, Alexandre Perez-Lebel, Ga\"el Varoquaux, and Lennart Purucker for fruitful discussions regarding this work.

This publication is part of the Chair ``Markets and Learning'', supported by Air Liquide, BNP PARIBAS ASSET MANAGEMENT Europe, EDF, Orange and SNCF, sponsors of the Inria Foundation.

This work received support from the French government, managed by the National Research Agency, under the France 2030 program with the reference ``PR[AI]RIE-PSAI'' (ANR-23-IACL-0008).

Funded by the European Union (ERC-2022-SYG-OCEAN-101071601). Views and opinions expressed are however those of the author(s) only and do not necessarily reflect those of the European Union or the European Research Council Executive Agency. Neither the European Union nor the granting authority can be held responsible for them.

\bibliographystyle{unsrtnat}
\bibliography{references}

\newpage
\onecolumn

\begin{appendices}

\listofappendices

\counterwithin{figure}{section}
\counterwithin{table}{section}

\crefalias{section}{appendix}
\crefalias{subsection}{appendix}

\newpage

\section{Calibration, Refinement, and Sharpness}
\label{sec:appendix:calref}

In this section we first prove our main theoretical result, the variational reformulation of proper loss decompositions.
We then discuss an alternative decomposition and present additional results.

\thmVariational*

\begin{proof}
    Let $g_C(q) \coloneqq \bbE[Y|f(X)=q]$, such that $C = \bbE[Y|f(X)] = g_C(f(X))$ almost surely,
    \begin{IEEEeqnarray*}{+rCl+s*}
        \calR_\ell(f) & = & \bbE[e_\ell(C)] = \bbE[\ell(C, C)] = \bbE[\ell(g_C(f(X)), \bbE[Y|f(X)])] & (by definition) \\
        & = & \bbE[\ell(g_C(f(X)), Y)] & (by \Cref{lemma:risk_second_argument}) \\
        & = & \min_g \bbE[\ell(g(f(X)), Y)]~. & (by \Cref{lemma:inf_cond_exp} with $Z = f(X)$)
    \end{IEEEeqnarray*}
    
    This shows the representation of the refinement error, and \Cref{lemma:inf_cond_exp} also shows that $g^*$ is a minimizer, which is almost surely unique. The representation of the calibration error follows directly from the calibration-refinement decomposition $\Risk(f) = \calR_\ell(f) + \calK_\ell(f)$ \citep{brocker2009}.

    Now, let $Z \in \Delta_k$ be the uniform distribution (to be interpreted as a constant random variable). Then, for $\bar p \coloneqq \bbE[Y]$
    \begin{IEEEeqnarray*}{+rCl+s*}
        e_\ell(\bbE[Y]) & = & \ell(\bbE[Y], \bbE[Y]) = \ell(\bar p, \bbE[Y]) \\
        & = & \bbE[\ell(\bar p, Y)] & ($\ell$ is affine in the second argument) \\
        & = & \min_g \bbE[\ell(g(Z), Y)] & (by \Cref{lemma:inf_cond_exp}) \\
        & = & \min_{c \text{ constant}} \bbE[\ell(c(f(X)), Y)]~.
    \end{IEEEeqnarray*}
    Again, \Cref{lemma:inf_cond_exp} shows that the minimizer is $c^*(q) = \bbE[Y|Z] = \bbE[Y]$.

    For the sharpness, we have
    \begin{IEEEeqnarray*}{+rCl+s*}
        \calS_\ell(f) & = & \bbE[d_\ell(\bbE[Y], C)] = \bbE[\ell(\bbE[Y], C)] - \bbE[\ell(C, C)] = \bbE[\ell(\bbE[Y], \bbE[Y|f(X)])] - \calR_\ell(f)
    \end{IEEEeqnarray*}
    by definition. We already showed above that $\calR_\ell(f) = \min_g \Risk(g \circ f)$. For the other term, we exploit that $\ell$ is affine in its second argument to obtain
    \begin{IEEEeqnarray*}{+rCl+x*}
        \bbE[\ell(\bbE[Y], \bbE[Y|f(X)])] & = & \ell(\bbE[Y], \bbE[\bbE[Y|f(X)]]) = e_\ell(\bbE[Y])~.
    \end{IEEEeqnarray*}
\end{proof}

\begin{remark}[Alternative decompositions by \citet{kull2015novel}] \label{rem:kullflach}
    \citet{kull2015novel} provide a slightly different calibration-refinement decomposition,
    \begin{equation*}
        \bbE[d_\ell(f(X), Y)] = \underbrace{\bbE[d_\ell(f(X), C)]}_{\calK_\ell(f)} + \underbrace{\bbE[d_\ell(C, Y)]}_{\eqqcolon \tilde \calR_\ell(f)}~.
    \end{equation*}
    The alternative refinement error
    \begin{IEEEeqnarray*}{rCl}
        \tilde \calR_\ell(f) = \bbE[d_\ell(C, Y)] & = & \bbE[\ell(C, Y)] - \bbE[\ell(Y, Y)] = \calR_\ell(f) - \bbE[\ell(Y, Y)]
    \end{IEEEeqnarray*}
    equals $\calR_\ell(f)$ if $\bbE[\ell(Y, Y)] = 0$, which is the case for logloss and Brier loss. In this case, the irreducible and grouping loss in \Cref{fig:extended_diagram} also equal the definitions of \citet{kull2015novel} in their decomposition
    \begin{equation*}
        \bbE[d_\ell(C, Y)] = \underbrace{\bbE[d_\ell(C, \bbE[Y|X])]}_{\text{grouping loss}} + \underbrace{\bbE[d_\ell(\bbE[Y|X], Y)]}_{\text{irreducible loss}}~.
    \end{equation*}
    Intuitively, the grouping loss describes the loss that occurs when $f(x_1) = f(x_2)$ but $\bbE[Y|X=x_1] \neq \bbE[Y|X=x_2]$.

    The reducible loss in \Cref{fig:extended_diagram} is called \emph{epistemic loss} in \citet{kull2015novel}.
\end{remark}

\begin{remark}[Extending the definitions of calibration, refinement, and sharpness]
    \Cref{thm:variational_extended} shows that for proper losses, calibration error is not only a divergence relative to a desired target ($C$), but also a potential reduction in loss through post-processing. 
    The latter perspective has been suggested by \citet{ferrer2024evaluating} without showing equivalence to the former perspective.
    On the one hand, an advantage of the former perspective is that it extends to divergences that are not induced by a proper loss (non-Bregman divergences) like the $L^1$ distance. On the other hand, an advantage of the latter perspective is that it can be used to extend the notion of calibration error, refinement error, and sharpness to non-proper losses (changing the desired target) or even non-probabilistic settings such as regression with MSE loss. Moreover, refinement error and sharpness can even be defined when $f(X)$ is not in the right space for the loss function, for example because it is a hidden-layer representation in a neural network. The variational formulation also allows one to define more related quantities by using various function classes $\calG$ and considering $\inf_{g \in \calG} \Risk(g \circ f)$.
    For example, considering the corresponding variant of calibration error for 1-Lipschitz functions, we obtain the post-processing gap from \citet{blasiok2024does}. In the binary case, using increasing functions, we obtain the best possible risk attainable by isotonic regression.
\end{remark}

\Cref{thm:variational_extended} also implies some further properties of calibration, refinement, and sharpness:
\begin{corollary} \label{cor:calref}
Let $k \geq 2$, let $\ell: \Delta_k \times \calY_k \to \bbR \cup \{\infty\}$ be a proper loss, and let $f: \calX \to \Delta_k$ be measurable, where $\Delta_k$ is equipped with the Borel $\sigma$-algebra.
    \begin{enumerate}[(a)]
        \item If $f$ is calibrated, then its population risk cannot be reduced by post-processing, i.e., $\Risk(f) = \inf_g \Risk(g \circ f)$. The converse holds if $\ell$ is strictly proper \citep{blasiok2024does}.
        \item Let $g: \Delta_k \to \Delta_k$ be measurable. Then, $\calR_\ell(g \circ f) \geq \calR_\ell(f)$, and $\calS_\ell(g \circ f) \leq \calS_\ell(f)$.
        \item Let $g: \Delta_k \to \Delta_k$ be measurable and injective. Then, $\calR_\ell(g \circ f) = \calR_\ell(f)$, and $\calS_\ell(g \circ f) = \calS_\ell(f)$. (A similar result is shown in Proposition 4.5 of \citet{gruber2022better}.)
        \item Suppose $f$ is injective with a measurable left inverse (see \Cref{thm:left_inverse} for a sufficient condition). Then, the grouping loss is zero, that is, $\calR_\ell(f) = \inf_{\tilde f} \calR_\ell(\tilde f)$ and $\calS_\ell(f) = \sup_{\tilde f} \calS_\ell(\tilde f)$.
    \end{enumerate}
\end{corollary}

\begin{proof}
    \leavevmode
    \begin{enumerate}[(a)]
        \item By definition, $f$ is calibrated iff $f(X) = \bbE[Y|f(X)]$ almost surely. By \Cref{thm:variational_extended}, $\Risk(g \circ f)$ is minimized by $g^*(q) = \bbE[Y|f(X)=q] = q$, where the equality holds $P_{f(X)}$-almost surely. Hence, $\Risk(f) = \inf_g \Risk(g \circ f)$. For the converse, by \Cref{thm:variational_extended}, if $\ell$ is strictly proper, the minimizer of $\Risk(g \circ f)$ is unique up to $P_{f(X)}$-null sets, hence $\bbE[Y|f(X)=q] = g^*(q) = q$ for $P_{f(X)}$-almost all $q$, hence $\bbE[Y|f(X)] = f(X)$ almost surely.
        \item Since $h \circ g$ is measurable for all measurable $h$, we have
        \begin{equation*}
            \calR_\ell(g \circ f) = \inf_h \Risk((h \circ g) \circ f) \geq \inf_h \Risk(h \circ f) = \calR_\ell(f)~.
        \end{equation*}
        The sharpness inequality follows analogously.
        \item Let $h^* \in \argmin_h \Risk(h \circ f)$. By \Cref{thm:left_inverse}, there exists a measurable left inverse $g^{-1}$ of $g$. Then,
        $\calR_\ell(g \circ f) \leq \Risk((h^* \circ g^{-1}) \circ (g \circ f)) = \Risk(h^* \circ f) = \calR_\ell(f)$, which, together with (b), yields the claim for the refinement. The argument for sharpness is analogous.
        \item This follows analogously to (c).
    \end{enumerate}
\end{proof}

\begin{lemma}[Existence of measurable left inverses] \label{thm:left_inverse}
    Let $(X, \calB_X), (Y, \calB_Y)$ be complete separable metric spaces with their associated Borel $\sigma$-algebra, and let $f: X \to Y$ be measurable and injective. Then, there exists $g: Y \to X$ measurable with $g \circ f = \operatorname{id}$.
\end{lemma}

\begin{proof}
    Using Corollary 15.2 by \cite{kechris_classical_1995}, $f$ is a Borel isomorphism from $X$ to $f(X)$, so its inverse $f^{-1}: f(X) \to X$ is measurable. But then, for a fixed $y_0 \in Y$,
    \begin{equation*}
        g: Y \to X, y \mapsto \begin{cases}
            f^{-1}(y) &, y \in f(X) \\
            y_0 &, y \not\in f(X)
        \end{cases}
    \end{equation*}
    is a measurable left inverse of $f$, proving the claim.
\end{proof}

\begin{lemma} \label{lemma:risk_second_argument}
    Let $(Z, Y) \in \calZ \times \calY_k$ be random variables with distribution $P$ and let $\ell$ be proper. Then,
    \begin{equation*}
        \bbE[\ell(g(Z), \bbE[Y|Z])] = \bbE[\ell(g(Z), Y)]~.
    \end{equation*}
\end{lemma}

\begin{proof}
    Since $\calY_k$ is a Radon space, there exists a regular conditional probability distribution $P_{Y|Z}$, using which we obtain
    \begin{IEEEeqnarray*}{+rCl+x*}
        \bbE[\ell(g(Z), \bbE[Y|Z])] & = & \int \ell(g(z), \int y P_{Y|Z}(\diff y|z)) \diff P_Z(z) \\
        & = & \iint \ell(g(z), y)  P_{Y|Z}(\diff y|z) \diff P_Z(z) & (since $\ell$ is affine in the second argument) \\
        & = & \int \ell(g(z), y) \diff P(y, z) \\
        & = & \bbE[\ell(g(Z), Y)]~.
    \end{IEEEeqnarray*}
\end{proof}

The following lemma, used in the proof of \Cref{thm:variational_extended}, provides a formal version of a well-known result.

\begin{lemma} \label{lemma:inf_cond_exp}
    Let $(Z, Y) \in \calZ \times \calY_k$ be random variables and let $\ell$ be proper. Then, for measurable $g: \calZ \to \calY$,
    \begin{IEEEeqnarray*}{+rCl+x*}
        \Risk(g) \coloneqq \bbE[\ell(g(Z), Y)] 
    \end{IEEEeqnarray*}
    is minimized by $g^*(z) \coloneqq \bbE[Y|Z=z]$. If $\ell$ is strictly proper, then $g^*$ is the unique minimizer up to $P_Z$-null sets.
\end{lemma}

\begin{proof}
    We have
    \begin{IEEEeqnarray*}{+rCl+s*}
        \bbE[\ell(g(Z), Y)] & = & \bbE[\ell(g(Z), \bbE[Y|Z])] & (by \Cref{lemma:risk_second_argument}) \\
        &\geq& \bbE[\ell(\bbE[Y|Z], \bbE[Y|Z])] & ($\ell$ is proper) \\
        &=& \bbE[\ell(\bbE[Y|Z], Y)]  & (by \Cref{lemma:risk_second_argument})  \\
        &=& \bbE[\ell(g^*(Z), Y)]~.
    \end{IEEEeqnarray*}

    If $\Risk(g) = \Risk(g^*)$, the inequality above can only be strict for $z$ in a $P_Z$-null set. If $\ell$ is strictly proper, then the inequality is violated whenever $g^*(z) \neq g(z)$.
\end{proof}

\section{Proofs for \Cref{section:GaussianDataModel}}
\label{appendix:ProofsGaussianDataModel}

Under the data model introduced in \Cref{section:GaussianDataModel}, $\mathbb{P}(Y=1|X=x) = \sigma({w^*}^\top  x)$ with $w^* = 2 \Sigma^{-1} \mu$ and $\sigma(x) = \frac{1}{1+e^{-x}}$.
\begin{proof}
\begin{align*}
    \mathbb{P}(Y \! = \! 1|X \! = \! x) &= \frac{\mathbb{P}(X=x|Y=1) \mathbb{P}(Y=1)}{\mathbb{P}(X=x)}\\
    &= \frac{\mathbb{P}(X=x|Y=1) \mathbb{P}(Y=1)}{\mathbb{P}(X=x|Y=1) \mathbb{P}(Y=1) + \mathbb{P}(X=x|Y=-1) \mathbb{P}(Y=-1)}\\
    &= \frac{\mathcal{N}(x | \mu, \Sigma)}{\mathcal{N}(x | \mu, \Sigma) + \mathcal{N}(x | \text{-}\mu, \Sigma)}\quad \text{using $\mathbb{P}(Y \! = \! 1) \! = \! \mathbb{P}(Y \! = \! \text{-}1)$ and our data model.}\\
    &= \sigma \Big(\frac{1}{2}(x^\top \Sigma^{-1}\mu + \mu^\top \Sigma^{-1}x + x^\top \Sigma^{-1}\mu + \mu^\top \Sigma^{-1}x) \Big) \quad \text{dividing by $\mathcal{N}(x | \mu, \Sigma)$.}\\
    &= \sigma(2\mu^\top \Sigma^{-1}x) = \sigma({w^*}^\top x)
\end{align*}
With $w^* = 2 \Sigma^{-1} \mu$.
\end{proof}

We now prove the following lemmas that we will use to demonstrate \Cref{theorem:CalRefError}.
\begin{lemma}\label{lemma:Ygivenf}
$\mathbb{P}(Y=1 | f(X)=\sigma(z)) = \sigma\Big(\frac{\langle w, w^*\rangle_\Sigma}{\|w\|_\Sigma^2} z \Big)$.
\end{lemma}

\begin{lemma}\label{lemma:distrf}
$\mathbb{P}(f(X)=\sigma(z)) = \frac{1}{2}\mathcal{N}( z | \frac{\langle w, w^*\rangle_\Sigma}{2}, \|w\|_\Sigma^2) + \frac{1}{2}\mathcal{N}( z | -\frac{\langle w, w^*\rangle_\Sigma}{2}, \|w\|_\Sigma^2)$.
\end{lemma}

\begin{proof}
\begin{align*}
    \mathbb{P}(Y=1 | f(X)=\sigma(z)) &= \mathbb{P}(Y=1 | \sigma(w^\top X)=\sigma(z))\\
    &= \mathbb{P}(Y=1 | w^\top X=z)\\
    &= \frac{\mathbb{P}(w^\top X=z | Y=1)\mathbb{P}(Y=1)}{\mathbb{P}(w^\top X=z)}\\
    &= \frac{\mathbb{P}(w^\top X=z | Y=1)\mathbb{P}(Y=1)}{\mathbb{P}(w^\top X=z | Y=1)\mathbb{P}(Y=1) + \mathbb{P}(w^\top X=z | Y=-1)\mathbb{P}(Y=-1)}\\
\intertext{Using that an affine transformation of a multivariate normal distribution is a normal distribution ($w^\top X \sim \mathcal{N}(\pm w^\top \mu, w^\top \Sigma w)$)}
    &= \frac{\mathcal{N}(z | w^\top \mu, w^\top \Sigma w)}{\mathcal{N}( z | w^\top \mu, w^\top \Sigma w) + \mathcal{N}(z | -w^\top \mu, w^\top \Sigma w)}\\
    &= \sigma\Bigg(\frac{(z+w^\top \mu)^2 - (z-w^\top \mu)^2}{2w^\top \Sigma w}\Bigg) = \sigma\Bigg(\frac{2 w^\top \mu}{w^\top \Sigma w} z \Bigg) = \sigma\Bigg(\frac{\langle w, w^*\rangle_\Sigma}{\|w\|_\Sigma^2} z \Bigg)
\end{align*}
Along the way, we derived the distribution of $f(X)$,
\begin{align*}
    \mathbb{P}(f(X)=\sigma(z)) &= \mathbb{P}(w^\top X=z | Y=1)\mathbb{P}(Y=1) + \mathbb{P}(w^\top X=z | Y=-1)\mathbb{P}(Y=-1)\\
    &= \frac{1}{2}\mathcal{N}(z | w^\top \mu, w^\top \Sigma w) + \frac{1}{2}\mathcal{N}(z | -w^\top \mu, w^\top \Sigma w)\\
    &= \frac{1}{2}\mathcal{N}(z | \frac{\langle w, w^*\rangle_\Sigma}{2}, \|w\|_\Sigma^2) + \frac{1}{2}\mathcal{N}(z | -\frac{\langle w, w^*\rangle_\Sigma}{2}, \|w\|_\Sigma^2) \, .
\end{align*}
\end{proof}

\thmCalRefError*

\begin{proof}
The refinement error can be written as $\mathcal{R}(w) = \mathbb{E}_{f(X)}[e_\ell(\mathbb{E}[Y|f(X)])]$.
Using \Cref{lemma:Ygivenf} and \Cref{lemma:distrf},
\begin{align*}
    \mathcal{R}(w) = &\frac{1}{2} \int e_\ell \Bigg(\sigma\Big(\frac{\langle  w, w^*\rangle_\Sigma}{\|w\|_\Sigma^2} z \Big)\Bigg) \mathcal{N}\Big(z \Big| \frac{\langle w, w^*\rangle_\Sigma}{2}, \|w\|_\Sigma^2\Big) dz\\
    &+  \frac{1}{2} \int e_\ell \Bigg(\sigma\Big(\frac{\langle w, w^*\rangle_\Sigma}{\|w\|_\Sigma^2} z \Big)\Bigg) \mathcal{N}\Big( z \Big| -\frac{\langle w, w^*\rangle_\Sigma}{2}, \|w\|_\Sigma^2\Big) dz \, .
\end{align*}
Taking $u = (z \! - \! \frac{\langle w, w^*\rangle_\Sigma}{2})/\|w\|_\Sigma$ in the first integral and $u = (z \! + \! \frac{\langle w, w^*\rangle_\Sigma}{2})/\|w\|_\Sigma$ in the second,
\begin{align*}
    \mathcal{R}(w) = &\frac{1}{2} \int e_\ell \Bigg(\sigma\Big(\frac{\langle w, w^*\rangle_\Sigma}{\|w\|_\Sigma}u + \frac{\langle w, w^*\rangle_\Sigma^2}{2 \|w\|_\Sigma^2}\Big)\Bigg) \mathcal{N}(u | 0, 1) du\\
    &+ \frac{1}{2} \int e_\ell \Bigg(\sigma\Big(\frac{\langle w, w^*\rangle_\Sigma}{\|w\|_\Sigma}u - \frac{\langle w, w^*\rangle_\Sigma^2}{2 \|w\|_\Sigma^2}\Big)\Bigg) \mathcal{N}(u | 0, 1) du \, .
\end{align*}
Since $\sigma(-x) = 1 - \sigma(x)$ and $e_\ell$ is symmetric around $\frac{1}{2}$ (by symmetry of $\ell$), the two terms are equal, which simplifies the expression,
\begin{equation*}
    \mathcal{R}(w) = \mathbb{E}_{z \sim \mathcal{N}(0,1)}\Big[ e_\ell \Big( \sigma\Big(\frac{\langle w, w^*\rangle_\Sigma}{\|w\|_\Sigma} \Big( z + \frac{\langle w, w^*\rangle_\Sigma}{2 \|w\|_\Sigma}\Big) \Big) \Big) \Big] \, .
\end{equation*}

Similarly for the calibration error
\begin{align*}
    \mathcal{K}(w) &= \mathbb{E}_{f(X)}[d_\ell(f(X), \mathbb{E}[Y|f(X)])]\\
    &= \frac{1}{2} \int d_\ell \Big(\sigma(z), \sigma\Big(\frac{\langle w, w^*\rangle_\Sigma}{\|w\|_\Sigma^2} z \Big)\Big) \mathcal{N}\Big(z \Big| \frac{\langle w, w^*\rangle_\Sigma}{2}, \|w\|_\Sigma^2 \Big) dz\\
    &+  \frac{1}{2} \int d_\ell \Big(\sigma(z), \sigma\Big(\frac{\langle w, w^*\rangle_\Sigma}{\|w\|_\Sigma^2} z \Big)\Big) \mathcal{N}\Big( z \Big| -\frac{\langle w, w^*\rangle_\Sigma}{2}, \|w\|_\Sigma^2\Big) dz\\
    &= \frac{1}{2} \int d_\ell \Big(\sigma\Big(\|w\|_\Sigma z + \frac{\langle w, w^*\rangle_\Sigma}{2}\Big), \sigma\Big(\frac{\langle w, w^*\rangle_\Sigma}{\|w\|_\Sigma} z + \frac{\langle w, w^*\rangle_\Sigma^2}{2\|w\|_\Sigma^2}\Big)\Big) \mathcal{N}(z | 0, 1) dz \\
    &+ \frac{1}{2} \int d_\ell \Big(\sigma\Big(\|w\|_\Sigma z - \frac{\langle w, w^*\rangle_\Sigma}{2}\Big), \sigma\Big(\frac{\langle w, w^*\rangle_\Sigma}{\|w\|_\Sigma} z - \frac{\langle w, w^*\rangle_\Sigma^2}{2\|w\|_\Sigma^2}\Big)\Big) \mathcal{N}(z | 0, 1) dz \, .
\end{align*}
Since $\sigma(-x) = 1 - \sigma(x)$ and $d_\ell(1-p, 1-q) = d_\ell(p, q)$ (by symmetry of $\ell$), the two terms are equal,
\begin{equation*}
    \mathcal{K}(w) = \mathbb{E}_{z\sim \mathcal{N}(0,1)}\Big[ d_\ell \Big( \sigma\Big( \|w\|_\Sigma \Big( z + \frac{\langle w, w^*\rangle_\Sigma}{2\|w\|_\Sigma} \Big) \Big), \sigma\Big( \frac{\langle w, w^*\rangle_\Sigma}{\|w\|_\Sigma} \Big( z + \frac{\langle w, w^*\rangle_\Sigma}{2\|w\|_\Sigma} \Big) \Big) \Big) \Big] \, .
\end{equation*}
\end{proof}

\thmTemperatureScaling*

\begin{proof}
$a_w = \frac{\langle w, w^*\rangle_\Sigma}{\|w\|_\Sigma}$ is invariant by rescaling. Given $s > 0$ and $w_s = sw$, we have that $a_{w_s} = \frac{\langle sw, w^*\rangle_\Sigma}{\|sw\|_\Sigma} = \frac{s\langle w, w^*\rangle_\Sigma}{s\|w\|_\Sigma} = a_w$. Refinement error in \Cref{theorem:CalRefError} depends on $w$ only via $a_w$, so the refinement is invariant by rescaling.
The calibration error can be written as
$$
    \mathcal{K}(w) = \mathbb{E}_{z\sim \mathcal{N}(0,1)}\Big[ d_\ell \Big( \sigma\Big( \|w\|_\Sigma \Big( z + \frac{a_w}{2} \Big) \Big), \sigma\Big( a_w \Big( z + \frac{a_w}{2} \Big) \Big) \Big) \Big] \, .
$$
Since $a_w$ is invariant by rescaling, taking $w_s = \frac{\langle w, w^*\rangle_\Sigma}{\|w\|^2_\Sigma} w$ yields
\begin{align*}
    \mathcal{K}(w_s) &= \mathbb{E}_{z\sim \mathcal{N}(0,1)}\Big[ d_\ell \Big( \sigma\Big( \frac{\langle w, w^*\rangle_\Sigma}{\|w\|^2_\Sigma} \|w\|_\Sigma \Big( z + \frac{a_w}{2} \Big) \Big), \sigma\Big( a_w \Big( z + \frac{a_w}{2} \Big) \Big) \Big) \Big]\\
    &= \mathbb{E}_{z\sim \mathcal{N}(0,1)}\Big[ d_\ell \Big( \sigma\Big( a_w \Big( z + \frac{a_w}{2} \Big) \Big), \sigma\Big( a_w \Big( z + \frac{a_w}{2} \Big) \Big) \Big) \Big]\\
    &=0 \quad \text{since $\ell$ is a proper score ($d_\ell(p,p) = 0$).}
\end{align*}

\end{proof}

\section{Proofs for \Cref{section:HighDimensionalLogisticRegression}} 
\label{appendix:BetaSpectral}

For this section, we rely heavily on the following result by \citet{MaiHighDimLR} on the limit distribution of the weight vector for regularized logistic regression, with regularization strength $\lambda$, when $n,p \rightarrow \infty$ with a fixed ratio $p/n \rightarrow r$,
$$
w_\lambda \sim \mathcal{N}\Big(\eta (\lambda I_p + \tau \Sigma)^{-1} \mu, \frac{\gamma}{n} (\lambda I_p + \tau \Sigma)^{-1} \Sigma (\lambda I_p + \tau \Sigma)^{-1} \Big) .
$$

\propAsymptoticDistributions*

\begin{proof}
$\langle w_\lambda, w^* \rangle_\Sigma$ follows a normal distribution with mean
$2\eta \mu^\top  (\lambda I_p + \tau \Sigma)^{-1} \mu$ and variance $\frac{4\gamma}{n} \mu^\top  (\lambda I_p + \tau \Sigma)^{-1} \Sigma (\lambda I_p + \tau \Sigma)^{-1} \mu$.
Looking at the variance first,
$$
    \mathbb{V}(\langle w_\lambda, w^* \rangle_\Sigma) = \frac{4\gamma}{n} \sum_{i=1}^p \frac{\sigma_i \mu_i^2}{(\lambda + \tau \sigma_i)^2} \, .
$$
Using independence between $\mu_i$ and $\sigma_i$ and the strong law of large numbers,
$$
    \mathbb{V}(\langle w_\lambda, w^* \rangle_\Sigma) \stackrel{p \to \infty}{\sim} \frac{4\gamma}{n} p \mathbb{E}[\mu_i^2] \mathbb{E}\Big[\frac{\sigma_i}{(\lambda + \tau \sigma_i)^2} \Big]
    = \frac{4 \gamma c^2}{n} \mathbb{E}\Big[\frac{\sigma_i}{(\lambda + \tau \sigma_i)^2} \Big] \xrightarrow{n \rightarrow \infty} 0 \, .
$$
So $\langle w_\lambda, w^* \rangle_\Sigma$ concentrates on its mean,
$$
    \mathbb{E}(\langle w_\lambda, w^* \rangle_\Sigma) = 2 \eta \sum_{i=1}^p \frac{\mu_i^2}{\lambda + \tau \sigma_i} \stackrel{p \to \infty}{\sim} 2\eta p \mathbb{E}[\mu_i^2] \mathbb{E}\Big[\frac{1}{\lambda + \tau \sigma_i}\Big] = 2 \eta c^2 \mathbb{E}\Big[\frac{1}{\lambda + \tau \sigma_i}\Big] \, .
$$
Studying $w_\lambda^\top  \Sigma w_\lambda$ requires a bit more work. Denoting $(\varepsilon_i)_{1 \leq i \leq p}$ i.i.d.\ $\mathcal{N}(0,1)$ random variables,
\begin{align*}
    w_\lambda^\top  \Sigma w_\lambda \sim \sum_{i=1}^p \sigma_i {w_\lambda}_i^2 &= \sum_{i=1}^p \sigma_i \Big(\frac{\eta \mu_i}{\lambda + \tau \sigma_i} + \varepsilon_i \sqrt{\frac{\gamma \sigma_i}{n(\lambda + \tau \sigma_i)^2}} \Big)^2 \\
    &= \sum_{i=1}^p  \frac{\eta^2 \mu_i^2 \sigma_i}{(\lambda + \tau \sigma_i)^2} + \sum_{i=1}^p \frac{2 \eta \mu_i \sqrt{\gamma} \sigma_i^{3/2}}{\sqrt{n}(\lambda + \tau \sigma_i)^2} \varepsilon_i + \sum_{i=1}^p \frac{\gamma \sigma_i^2}{n(\lambda + \tau \sigma_i)^2} \varepsilon_i^2\\
    &\stackrel{p \to \infty}{\sim} p \eta^2 \frac{c^2}{p} \mathbb{E}\Big[\frac{\sigma}{(\lambda + \tau \sigma)^2}\Big] + 0 + p \frac{\gamma}{n} \mathbb{E}\Big[ \frac{\sigma^2}{(\lambda+\tau\sigma)^2} \Big]\\
    &= \eta^2 c^2 \mathbb{E}\Big[\frac{\sigma}{(\lambda + \tau \sigma)^2}\Big] + \gamma r \mathbb{E}\Big[ \frac{\sigma^2}{(\lambda+\tau\sigma)^2} \Big] \, .
\end{align*}
\end{proof}

\propOptimalErrorRate*

\begin{proof}
\begin{align*}
        \err(w^*) = \Phi(-\frac{\langle w^*,w^* \rangle_\Sigma}{2\|w^*\|_\Sigma}) = \Phi(-\frac{\sqrt{{w^*}^\top  \Sigma w^*}}{2}) = & \Phi(-\sqrt{\mu^\top  \Sigma^{-1} \mu}) \\
        \xrightarrow{a.s.} & \Phi(-\sqrt{p \mathbb{E}[\mu^2] \mathbb{E}[\frac{1}{\sigma}]}) = \Phi(-c \sqrt{\mathbb{E}[\sigma^{-1}]})
\end{align*}
\end{proof}

\paragraph{Deriving formulas for $\err(w^*)$, $\langle w_\lambda, w^* \rangle_\Sigma$ and $\|w_\lambda \|_\Sigma$}
\begin{align*}
    \mathbb{E}[\sigma^{-1}] &= \int_0^1 \frac{1}{u + \varepsilon} \frac{u^{\alpha-1} (1-u)^{\beta-1}}{B(\alpha, \beta)} du\\
    &= \frac{1}{\varepsilon B(\alpha, \beta)} \int_0^1 (1 + \frac{u}{\varepsilon})^{-1} u^{\alpha-1} (1-u)^{\beta-1} du
    = \frac{{}_2 F_1(1, \alpha, \alpha+\beta, -\frac{1}{\varepsilon})}{\varepsilon} \, .
\end{align*}

Altogether, $
\err(w^*) = \Phi \Big(- c \sqrt{ \frac{1}{\varepsilon} {}_2 F_1(1, \alpha, \alpha+\beta, -\frac{1}{\varepsilon})} \Big)$.

\begin{lemma}\label{lemma:BetaFormulas}
Denoting ${}_2 F_1$ the hypergeometric function, $\mathcal{B}$ the Beta distribution and $B$ the beta function,
\begin{align*}
    \mathbb{E}_{u \sim \mathcal{B}(\alpha, \beta)}\Big[ \frac{1}{\lambda + \tau(\varepsilon + u)} \Big] &= \int_0^1 \frac{1}{\lambda + \tau(\varepsilon + u)} \frac{u^{\alpha-1} (1-u)^{\beta-1}}{B(\alpha, \beta)} du\\
    &= \frac{1}{(\lambda + \tau \varepsilon) B(\alpha, \beta)} \int_0^1 (1 + \frac{\tau}{\lambda + \tau \varepsilon} u)^{-1} u^{\alpha-1} (1-u)^{\beta-1} du\\
    &= \frac{{}_2 F_1(1, \alpha, \alpha + \beta, -(\frac{\lambda}{\tau}+\varepsilon)^{-1})}{\lambda + \tau \varepsilon} \, ,
\end{align*}
\begin{align*}
    \mathbb{E}_{u \sim \mathcal{B}(\alpha, \beta)}\Big[ \frac{1}{(\lambda + \tau(\varepsilon + u))^2} \Big] &= \int_0^1 \frac{1}{(\lambda + \tau(\varepsilon + u))^2} \frac{u^{\alpha-1} (1-u)^{\beta-1}}{B(\alpha, \beta)} du\\
    &= \frac{1}{(\lambda + \tau \varepsilon)^2 B(\alpha, \beta)} \int_0^1 (1 + \frac{\tau}{\lambda + \tau \varepsilon} u)^{-2} u^{\alpha-1} (1-u)^{\beta-1} du\\
    &= \frac{{}_2 F_1(2, \alpha, \alpha+\beta, -(\frac{\lambda}{\tau}+\varepsilon)^{-1})}{(\lambda + \tau \varepsilon)^2} \, ,
\end{align*}
\begin{align*}
    \mathbb{E}_{u \sim \mathcal{B}(\alpha, \beta)}\Big[ \frac{u}{(\lambda + \tau(\varepsilon + u))^2} \Big] &= \int_0^1 \frac{u}{(\lambda + \tau(\varepsilon + u))^2} \frac{u^{\alpha-1} (1-u)^{\beta-1}}{B(\alpha, \beta)} du\\
    &= \frac{1}{(\lambda + \tau \varepsilon)^2 B(\alpha, \beta)} \int_0^1 (1 + \frac{\tau}{\lambda + \tau \varepsilon} u)^{-2} u^{\alpha} (1-u)^{\beta-1} du\\
    &= \frac{B(\alpha+1, \beta)}{B(\alpha, \beta)(\lambda + \tau \varepsilon)^2} {}_2 F_1(2, \alpha+1, \alpha+\beta+1, -(\frac{\lambda}{\tau}+\varepsilon)^{-1}) \, ,
\end{align*}
\begin{align*}
    \mathbb{E}_{u \sim \mathcal{B}(\alpha, \beta)}\Big[ \frac{u^2}{(\lambda + \tau(\varepsilon + u))^2} \Big] &= \int_0^1 \frac{u^2}{(\lambda + \tau(\varepsilon + u))^2} \frac{u^{\alpha-1} (1-u)^{\beta-1}}{B(\alpha, \beta)} du\\
    &= \frac{1}{(\lambda + \tau \varepsilon)^2 B(\alpha, \beta)} \int_0^1 (1 + \frac{\tau}{\lambda + \tau \varepsilon} u)^{-2} u^{\alpha+1} (1-u)^{\beta-1} du\\
    &= \frac{B(\alpha+2, \beta)}{B(\alpha, \beta)(\lambda + \tau \varepsilon)^2} {}_2 F_1(2, \alpha+2, \alpha+\beta+2, -(\frac{\lambda}{\tau}+\varepsilon)^{-1}) \, .
\end{align*}
\end{lemma}

Using the equations in \Cref{lemma:BetaFormulas} and the linearity of the expectation we get that
$$
\langle w_\lambda, w^* \rangle_\Sigma \xrightarrow{(n,p) \rightarrow \infty} 2 \eta c^2 \mathbb{E}_{u \sim \mathcal{B}(\alpha, \beta)}\Big[\frac{1}{\lambda + \tau (\varepsilon + u)}\Big] = 2 \eta c^2 \frac{{}_2 F_1(1, \alpha, \alpha + \beta, -(\frac{\lambda}{\tau}+\varepsilon)^{-1})}{\lambda + \tau \varepsilon} \, .
$$
Similarly,
\begin{align*}
    \|w_\lambda\|_\Sigma^2 &\xrightarrow{(n,p) \rightarrow \infty} \eta^2 c^2 \mathbb{E}_{u \sim \mathcal{B}(\alpha, \beta)}\Big[\frac{\varepsilon + u}{(\lambda+\tau (\varepsilon + u))^2}\Big] + \gamma r\mathbb{E}_{u \sim \mathcal{B}(\alpha, \beta)}\Big[ \frac{(\varepsilon + u)^2}{(\lambda + \tau (\varepsilon + u))^2} \Big]\\
    =& \frac{\eta^2 c^2}{(\lambda \! + \! \tau \varepsilon)^2} \Big( \! \varepsilon {}_2 F_1(2, \alpha, \alpha \! + \! \beta, -(\frac{\lambda}{\tau} \! + \! \varepsilon)^{-1}) + \frac{B(\alpha \! + \! 1, \beta)}{B(\alpha, \beta)} {}_2 F_1(2, \alpha \! + \! 1, \alpha \! + \! \beta \! + \! 1, -(\frac{\lambda}{\tau} \! + \! \varepsilon)^{-1}) \! \Big)\\
    &+ \frac{\gamma r}{(\lambda + \tau \varepsilon)^2} \Big( \! \frac{B(\alpha+2, \beta)}{B(\alpha, \beta)} {}_2 F_1(2, \alpha+2, \alpha+\beta+2, -(\frac{\lambda}{\tau}+\varepsilon)^{-1})\\
    &+ 2 \varepsilon \frac{B(\alpha \! + \! 1, \beta)}{B(\alpha, \beta)} {}_2 F_1(2, \alpha \! + \! 1, \alpha \! + \! \beta \! + \! 1, -(\frac{\lambda}{\tau} \! + \! \varepsilon)^{-1}) + \varepsilon^2  {}_2 F_1(2, \alpha, \alpha \! + \! \beta, -(\frac{\lambda}{\tau} \! + \! \varepsilon)^{-1}) \! \Big) \, .
\end{align*}

Solving the non-linear system of \citet{MaiHighDimLR} also requires computing the quantity $\kappa$, which depends on the auxiliary variable $e$ of that system (we use the notation of \citet{MaiHighDimLR}),
\begin{align*}
    &\kappa \xrightarrow{(n,p) \rightarrow \infty} r\mathbb{E}_{u \sim \mathcal{B}(\alpha, \beta)}\Big[ \frac{u + \varepsilon}{\lambda - e(\varepsilon + u)} \Big]\\
    &= \! r \Big( \! \int_0^1 \frac{u}{\lambda \! - \! e(\varepsilon \! + \! u)} \frac{u^{\alpha-1} (1 \! - \! u)^{\beta-1}}{B(\alpha, \beta)} du + \varepsilon \! \int_0^1 \! \frac{1}{\lambda \! - \! e(\varepsilon \! + \! u)} \frac{u^{\alpha-1} (1 \! - \! u)^{\beta-1}}{B(\alpha, \beta)} du \! \Big)\\
    &= \! \frac{r}{B(\alpha, \beta)(\lambda \! - \! e \varepsilon)} \Big( \! \int_0^1 (1 \! - \! \frac{e}{\lambda \! - \! e \varepsilon} u)^{-1} u^{\alpha} (1 \! - \! u)^{\beta-1} du + \varepsilon \! \int_0^1 (1 \! - \! \frac{e}{\lambda \! - \! e \varepsilon} u)^{-1} u^{\alpha-1} (1 \! - \! u)^{\beta-1} du \! \Big)\\
    &= \! \frac{r}{\lambda - e \varepsilon} \Big( \! \frac{B(\alpha+1, \beta)}{B(\alpha, \beta)} {}_2 F_1(1, \alpha+1, \alpha+\beta+1, (\frac{\lambda}{e}-\varepsilon)^{-1}) + \varepsilon \cdot {}_2 F_1(1, \alpha, \alpha+\beta, (\frac{\lambda}{e}-\varepsilon)^{-1}) \! \Big) \, .
\end{align*}

We provide an efficient implementation of the non-linear system solver with our particular data model and functions to compute the resulting calibration and refinement errors\footnote{\url{https://github.com/eugeneberta/RefineThenCalibrate-Theory}}.

\section{Temperature scaling implementation and TS-refinement}
\label{appendix:TemperatureScaling}

For the logloss, the objective $\calL(\beta) = \sum_{(x, y) \in \mathcal{C}} -y^\top \log(\softmax(\beta \log(f(x))))$ is convex, which follows from the fact that $-\log(\softmax(x)) = \logsumexp(x) - x$ is convex \citep{boyd2004convex}.
We use 30 bisection steps to approximate $b^* \coloneqq \argmin_{b \in [-16, 16]} \calL(\exp(b))$ and then find $\beta^* \coloneqq \exp(b^*)$.
If even faster implementations are desired, more advanced line search for convex functions \citep{orseau2023line}, second-order methods, stochastic optimization, or warm-starting could be helpful.
To guard against overfitting in cases with 100\% validation accuracy, where bisection can efficiently approach $\beta^* = \infty$, we introduce a form of Laplace smoothing (LS) as 
$g(p) \coloneqq \frac{\Ncal}{\Ncal+1} g_{\beta^*}(p) + \frac{1}{\Ncal+1} u$,
where $u$ is the uniform distribution and $\Ncal$ is the number of samples in the calibration set. This version is used for our main experiments except for \Cref{fig:calib-benchmark}.

\paragraph{Extended comparison} \Cref{fig:calib-benchmark-extra} shows a comparison of more post-hoc calibration methods. While \Cref{fig:calib-benchmark} uses mean logloss for simplicity, this can over-emphasize multi-class datasets while imbalanced datasets get less weight. To compensate for this, in \Cref{fig:calib-benchmark-extra} we divide the loss (logloss or Brier loss) for each dataset by the loss of the best constant predictor, which can be written as $e_\ell(\bbE[Y])$.

Our inclusion of Laplace smoothing (LS) in TS brings a small benefit for logloss. We also compare our results to isotonic regression, where we include LS to avoid infinite logloss since IR can predict a probability of zero.\footnote{IVAP \citep{vovk2015large} is a variant of isotonic regression that does not predict zero probabilities but we found it to be roughly two orders of magnitude slower than isotonic regression, so we did not consider it very attractive as a component of a refinement estimator.} While the IR implementation from scikit-learn \citep{pedregosa2011scikit} is almost as fast as our TS implementation, it performs considerably worse on our benchmark. For Brier loss, the benefit of TS in general is less pronounced, and only good implementations improve results compared to no post-hoc calibration.

\begin{figure}[tb]
    \centering
    \includegraphics[width=0.48\linewidth]{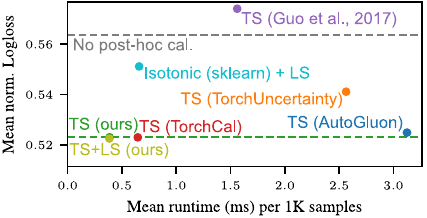}
    \includegraphics[width=0.48\linewidth]{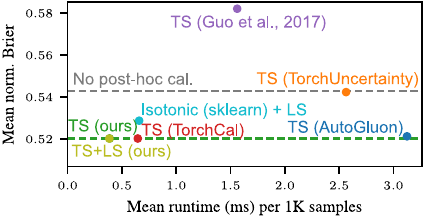}
    \caption{\emph{Runtime vs.\ mean benchmark scores of different TS implementations and isotonic regression for logloss (left) and Brier loss (right).} Runtimes are averaged over validation sets with at least 10K samples. The evaluation is on XGBoost models trained with default parameters, using the epoch with the best validation accuracy. The $y$-axis shows the mean of normalized logloss (left) or Brier loss (right), where normalizing means dividing by the loss $e_\ell(\bbE[Y])$ of the best constant predictor on each dataset.}
    \label{fig:calib-benchmark-extra}
\end{figure}

\paragraph{Validation} We are interested in minimizing the population risk of calibrated estimators. To faithfully estimate the population risk, it is possible to use cross-validation or the holdout method, fitting the calibrator on a part of the validation set while using its loss on the remaining part as a refinement estimate. However, we experimentally observe in \Cref{fig:tabular_more_metrics} that fitting and evaluating temperature scaling on the entire validation set works equally well, perhaps because it only fits a single parameter. We use the latter approach in the following since it is more efficient compared to cross-validation.

\section{Computer vision experiments} 
\label{appendix:ComputerVision}

We train a ResNet-18 \citep{he2016} and WideResNet \citep{zagoruyko2016} on CIFAR-10, CIFAR-100 and SVHN datasets \citep{krizhevsky2009, netzer2011}. We train for 300 epochs using SGD with momentum $=0.9$, weight decay $=10^{-4}$ and a learning rate scheduler that divides the learning rate by ten every 100 epochs. We use random cropping ($32\times32$ crops after padding images by $4$ pixels), horizontal flips (PyTorch default) and cutout \citep{devries2017} (one $16\times16$ hole per image during training). We train each model ten times on each dataset. 10\% of the training set is kept for validation. We select the best epoch based on different metrics evaluated on the validation set: logloss, Brier score, accuracy and logloss after temperature scaling (TS-refinement). In \Cref{tab:ResNet-18} (ResNet-18) and \Cref{tab:WideResNet} (WideResNet), we report average logloss, Brier score, accuracy, ECE and smooth ECE \citep{blasiok2024smooth} obtained for each stopping metric. \Cref{figure:TrainingCurve} is generated with the same training recipe with a cosine annealing scheduler instead to avoid big jumps in the loss landscape and no weight decay to exacerbate the effect of calibration overfitting. See \url{github.com/eugeneberta/RefineThenCalibrate-Vision}.

\begin{table}[h]
\centering
\caption{Benchmark results for ResNet-18. We report means obtained over the 10 runs and 95\% confidence intervals computed using the t-distribution. Stopping metrics are evaluated on the validation set (10\% of training data). All metrics are reported for the best model as selected by the corresponding stopping metric, after TS (fitted using the validation set).}
\resizebox{\textwidth}{!}{
\begin{tabular}{l l c c c c c}
\toprule
& & Logloss & Brier & Accuracy & ECE & smECE \\
Dataset & Stopping metric & & & \\
\midrule
\multirow{4}{*}{\footnotesize CIFAR-10} & Logloss & $0.161$ \footnotesize{$\pm 0.004$} & $0.080$ \footnotesize{$\pm 0.002$} & $94.6\%$ \footnotesize{$\pm 0.2$} & $\bf 0.006$ \footnotesize{$\pm 0.001$}& $\bf 0.008$ \footnotesize{$\pm 0.001$} \\
                          & Brier score & $\bf 0.154$ \footnotesize{$\pm 0.005$} & $\bf 0.073$ \footnotesize{$\pm 0.003$} & $\bf 95.1\%$ \footnotesize{$\pm 0.2$} & $0.007$ \footnotesize{$\pm 0.001$}& $\bf 0.008$ \footnotesize{$\pm 0.001$} \\
                          & Accuracy & $0.155$ \footnotesize{$\pm 0.005$} & $0.074$ \footnotesize{$\pm 0.003$} & $\bf 95.1\%$ \footnotesize{$\pm 0.2$} & $0.007$ \footnotesize{$\pm 0.001$}& $0.009$ \footnotesize{$\pm 0.001$} \\
                          & TS-refinement & $\bf 0.154$ \footnotesize{$\pm 0.005$} & $0.074$ \footnotesize{$\pm 0.003$} & $\bf 95.1\%$ \footnotesize{$\pm 0.2$} & $0.007$ \footnotesize{$\pm 0.002$}& $\bf 0.008$ \footnotesize{$\pm 0.001$} \\
\midrule
\multirow{4}{*}{\footnotesize CIFAR-100} & Logloss & $1.015$ \footnotesize{$\pm 0.008$} & $0.367$ \footnotesize{$\pm 0.003$} & $\bf 73.2\%$ \footnotesize{$\pm 0.3$} & $0.019$ \footnotesize{$\pm 0.002$}& $0.019$ \footnotesize{$\pm 0.001$} \\
                          & Brier score & $1.014$ \footnotesize{$\pm 0.007$} & $0.367$ \footnotesize{$\pm 0.002$} & $73.1\%$ \footnotesize{$\pm 0.3$} & $0.020$ \footnotesize{$\pm 0.002$}& $0.019$ \footnotesize{$\pm 0.001$} \\
                          & Accuracy & $1.001$ \footnotesize{$\pm 0.012$} & $\bf 0.366$ \footnotesize{$\pm 0.003$} & $73.1\%$ \footnotesize{$\pm 0.3$} & $0.018$ \footnotesize{$\pm 0.001$}& $0.018$ \footnotesize{$\pm 0.001$} \\
                          & TS-refinement & $\bf 0.983$ \footnotesize{$\pm 0.008$} & $\bf 0.366$ \footnotesize{$\pm 0.003$} & $73.0\%$ \footnotesize{$\pm 0.2$} & $\bf 0.016$ \footnotesize{$\pm 0.001$}& $\bf 0.016$ \footnotesize{$\pm 0.001$} \\
\midrule
\multirow{4}{*}{\footnotesize SVHN} & Logloss & $0.120$ \footnotesize{$\pm 0.003$} & $\bf 0.048$ \footnotesize{$\pm 0.001$} & $96.9\%$ \footnotesize{$\pm 0.1$} & $\bf 0.006$ \footnotesize{$\pm 0.001$}& $\bf 0.007$ \footnotesize{$\pm 0.001$} \\
                          & Brier score & $\bf 0.118$ \footnotesize{$\pm 0.002$} & $\bf 0.048$ \footnotesize{$\pm 0.001$} & $\bf 97.0\%$ \footnotesize{$\pm 0.1$} & $\bf 0.006$ \footnotesize{$\pm 0.001$}& $\bf 0.007$ \footnotesize{$\pm 0.001$} \\
                          & Accuracy & $0.124$ \footnotesize{$\pm 0.005$} & $0.049$ \footnotesize{$\pm 0.002$} & $96.9\%$ \footnotesize{$\pm 0.1$} & $\bf 0.006$ \footnotesize{$\pm 0.001$}& $0.008$ \footnotesize{$\pm 0.001$} \\
                          & TS-refinement & $0.119$ \footnotesize{$\pm 0.002$} & $\bf 0.048$ \footnotesize{$\pm 0.001$} & $\bf 97.0\%$ \footnotesize{$\pm 0.1$} & $\bf 0.006$ \footnotesize{$\pm 0.001$}& $\bf 0.007$ \footnotesize{$\pm 0.001$} \\
\bottomrule
\end{tabular}}
\label{tab:ResNet-18}
\end{table}

\begin{table}[h]
\centering
\small
\caption{Benchmark results for WideResNet. We report means obtained over the 10 runs and 95\% confidence intervals computed using the t-distribution. Stopping metrics are evaluated on the validation set (10\% of training data). All metrics are reported for the best model as selected by the corresponding stopping metric, after TS (fitted using the validation set).}
\resizebox{\textwidth}{!}{
\begin{tabular}{l l c c c c c}
\toprule
& & Logloss & Brier & Accuracy & ECE & smECE \\
Dataset & Stopping metric & & & \\
\midrule
\multirow{4}{*}{\footnotesize CIFAR-10} & Logloss & $0.123$ \footnotesize{$\pm 0.004$} & $0.061$ \footnotesize{$\pm 0.002$} & $95.9\%$ \footnotesize{$\pm 0.2$} & $\bf 0.006$ \footnotesize{$\pm 0.001$}& $\bf 0.008$ \footnotesize{$\pm 0.000$} \\
                          & Brier score & $0.121$ \footnotesize{$\pm 0.002$} & $\bf 0.056$ \footnotesize{$\pm 0.001$} & $\bf 96.4\%$ \footnotesize{$\pm 0.1$} & $\bf 0.006$ \footnotesize{$\pm 0.001$}& $\bf 0.008$ \footnotesize{$\pm 0.001$} \\
                          & Accuracy & $0.122$ \footnotesize{$\pm 0.002$} & $\bf 0.056$ \footnotesize{$\pm 0.001$} & $\bf 96.4\%$ \footnotesize{$\pm 0.1$} & $\bf 0.006$ \footnotesize{$\pm 0.001$}& $\bf 0.008$ \footnotesize{$\pm 0.001$} \\
                          & TS-refinement & $\bf 0.120$ \footnotesize{$\pm 0.002$} & $\bf 0.056$ \footnotesize{$\pm 0.001$} & $96.3\%$ \footnotesize{$\pm 0.1$} & $0.007$ \footnotesize{$\pm 0.001$}& $\bf 0.008$ \footnotesize{$\pm 0.001$} \\
\midrule
\multirow{4}{*}{\footnotesize CIFAR-100} & Logloss & $0.798$ \footnotesize{$\pm 0.007$} & $0.300$ \footnotesize{$\pm 0.002$} & $78.4\%$ \footnotesize{$\pm 0.3$} & $0.020$ \footnotesize{$\pm 0.002$}& $0.020$ \footnotesize{$\pm 0.002$} \\
                          & Brier score & $0.803$ \footnotesize{$\pm 0.003$} & $\bf 0.298$ \footnotesize{$\pm 0.001$} & $\bf 78.8\%$ \footnotesize{$\pm 0.1$} & $0.021$ \footnotesize{$\pm 0.002$}& $0.020$ \footnotesize{$\pm 0.001$} \\
                          & Accuracy & $0.791$ \footnotesize{$\pm 0.011$} & $0.300$ \footnotesize{$\pm 0.002$} & $78.6\%$ \footnotesize{$\pm 0.2$} & $0.020$ \footnotesize{$\pm 0.003$}& $0.019$ \footnotesize{$\pm 0.002$} \\
                          & TS-refinement & $\bf 0.771$ \footnotesize{$\pm 0.004$} & $\bf 0.298$ \footnotesize{$\pm 0.002$} & $78.5\%$ \footnotesize{$\pm 0.2$} & $\bf 0.016$ \footnotesize{$\pm 0.002$}& $\bf 0.016$ \footnotesize{$\pm 0.001$} \\
\midrule
\multirow{4}{*}{\footnotesize SVHN} & Logloss & $\bf 0.109$ \footnotesize{$\pm 0.002$} & $0.043$ \footnotesize{$\pm 0.001$} & $\bf 97.3\%$ \footnotesize{$\pm 0.1$} & $\bf 0.007$ \footnotesize{$\pm 0.001$}& $\bf 0.008$ \footnotesize{$\pm 0.001$} \\
                          & Brier score & $0.110$ \footnotesize{$\pm 0.003$} & $\bf 0.042$ \footnotesize{$\pm 0.001$} & $\bf 97.3\%$ \footnotesize{$\pm 0.1$} & $\bf 0.007$ \footnotesize{$\pm 0.001$}& $\bf 0.008$ \footnotesize{$\pm 0.001$} \\
                          & Accuracy & $0.112$ \footnotesize{$\pm 0.002$} & $\bf 0.042$ \footnotesize{$\pm 0.001$} & $\bf 97.3\%$ \footnotesize{$\pm 0.1$} & $0.008$ \footnotesize{$\pm 0.001$}& $0.009$ \footnotesize{$\pm 0.001$} \\
                          & TS-refinement & $\bf 0.109$ \footnotesize{$\pm 0.001$} & $\bf 0.042$ \footnotesize{$\pm 0.000$} & $\bf 97.3\%$ \footnotesize{$\pm 0.1$} & $\bf 0.007$ \footnotesize{$\pm 0.001$}& $\bf 0.008$ \footnotesize{$\pm 0.001$} \\
\bottomrule
\end{tabular}}
\label{tab:WideResNet}
\end{table}

\section{Tabular experiments}
\label{sec:appendix:tabular}

In the following, we provide more details and analyses for our tabular experiments. The experiments took around 40 hours to run on a workstation with a 32-core AMD Threadripper PRO 3975WX CPU and four NVIDIA RTX 3090 GPUs. The benchmark code is available at \url{github.com/dholzmueller/pytabkit}.

\subsection{Dataset size dependency}

\Cref{fig:gap_vs_ds_size} shows how the relative performance of refinement-based stopping vs logloss stopping depends on the dataset size. While the situation for small datasets is very noisy, for larger datasets the deterioration in the worst case is much less severe than the improvements in the best case. The advantages of TS-refinement become apparent at roughly 10K samples, which corresponds to a validation set size of 1,600. For XGB and MLP, there is one large outlier dataset where refinement-based stopping and logloss stopping both achieve low loss but the loss ratio is high.

\begin{figure}
    \centering
    \begin{minipage}{0.48\textwidth}
\centering
\includegraphics[width=\textwidth]{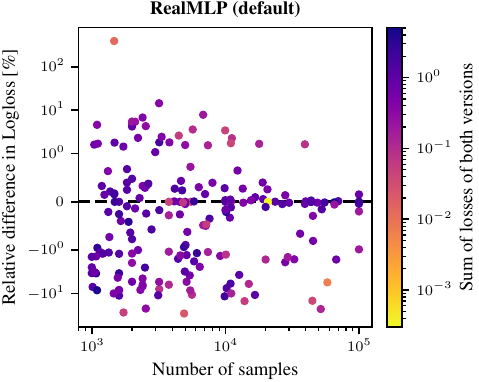}
\includegraphics[width=\textwidth]{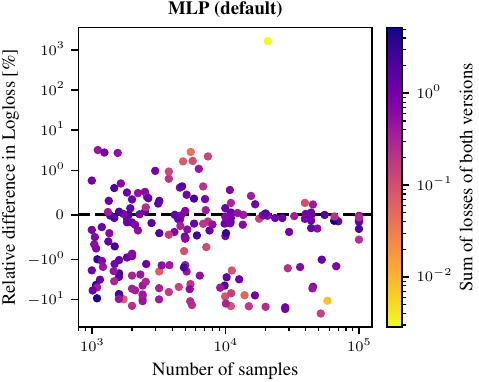}
\includegraphics[width=\textwidth]{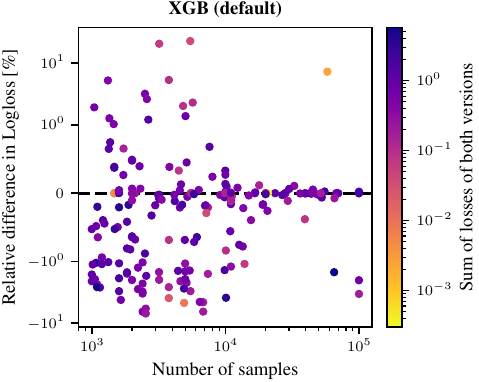}
    \end{minipage}
    \hspace{0.2cm}
    \begin{minipage}{0.48\textwidth}
\centering
\includegraphics[width=\textwidth]{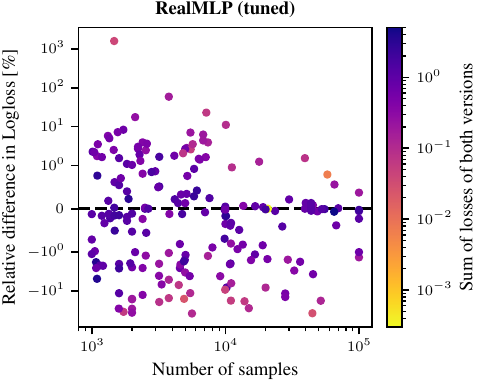}
\includegraphics[width=\textwidth]{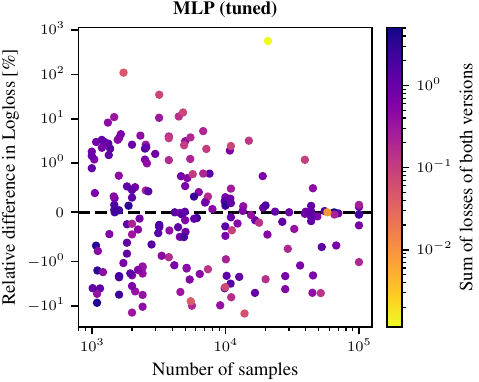}
\includegraphics[width=\textwidth]{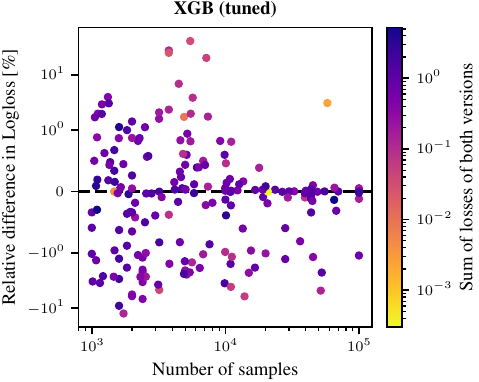}
    \end{minipage}
    \caption{\emph{Relative differences in logloss of using TS-refinement vs.\ Logloss for selecting the best epoch and hyperparameters.} Each method applies temperature scaling on the final model. Each dot represents one dataset. Values below zero mean that TS-refinement performs better. A light color indicates datasets where methods achieve very low loss.}
    \label{fig:gap_vs_ds_size}
\end{figure}

\subsection{Other stopping metrics}

In \Cref{fig:tabular_more_metrics}, we show results for more stopping metrics like AUROC, Brier loss, Brier loss after temperature scaling, and a 5-fold cross-validation version of TS-refinement from \Cref{appendix:TemperatureScaling}. TS-refinement and its more expensive cross-validation version perform best. Out of the other metrics, Brier loss before or after temperature scaling performs comparably to logloss, while AUROC performs worse than logloss for RealMLP. \Cref{fig:tabular_more_metrics_brier} shows the same stopping metrics on the Brier loss instead. Here, stopping on Brier loss and refinement-based metrics perform well.

\begin{figure}
    \centering
    \includegraphics[width=\linewidth]{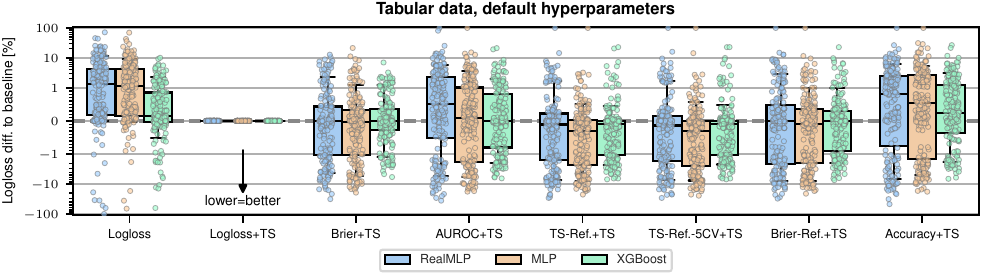}
    \caption{\emph{Relative differences in logloss for different stopping methods compared to Logloss+TS.} Brier-Ref.\ refers to Brier loss after temperature scaling. AUROC uses the one-vs-rest version for multi-class datasets. TS-Ref.-5CV is the metric from \Cref{appendix:TemperatureScaling} computing the out-of-fold loss after temperature scaling in a five-fold cross-validation setup. Methods labeled ``+TS'' use temperature scaling on the final model. Each dot represents one dataset from \citet{ye2024closer}, using only the 65 datasets with at least 10K samples. Percentages are clipped to $[-100, 100]$. Box-plots show the 10\%, 25\%, 50\%, 75\%, and 90\% quantiles.}
    \label{fig:tabular_more_metrics}
\end{figure}

\begin{figure}
    \centering
    \includegraphics[width=\linewidth]{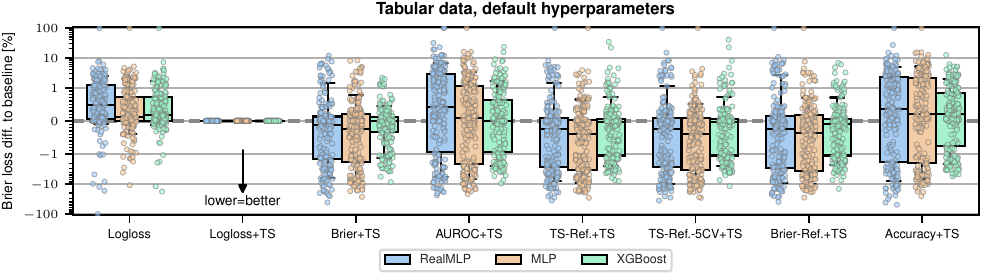}
    \caption{\emph{Relative differences in \emph{Brier loss} for different stopping methods compared to Logloss+TS.} Brier-Ref.\ refers to Brier loss after temperature scaling. AUROC uses the one-vs-rest version for multi-class datasets. TS-Ref.-5CV is the metric from \Cref{appendix:TemperatureScaling} computing the out-of-fold loss after temperature scaling in a five-fold cross-validation setup. Methods labeled ``+TS'' use temperature scaling on the final model. Each dot represents one dataset from \citet{ye2024closer}, using only the 65 datasets with at least 10K samples. Percentages are clipped to $[-100, 100]$. Box-plots show the 10\%, 25\%, 50\%, 75\%, and 90\% quantiles.}
    \label{fig:tabular_more_metrics_brier}
\end{figure}

\subsection{Effect on other metrics}

\Cref{fig:tabular_accuracy} shows that tuning TS-refinement can, on average, improve downstream accuracy compared to tuning logloss, and sometimes even compared to tuning accuracy directly. The same holds for AUROC in \Cref{fig:tabular_auroc}, where tuning accuracy often performs very badly. Performing temperature scaling does not affect downstream accuracy, but can affect downstream AUROC on multi-class datasets.

\begin{figure}[tb]
    \centering
    \begin{minipage}{0.48\textwidth}
\centering
\includegraphics[width=\textwidth]{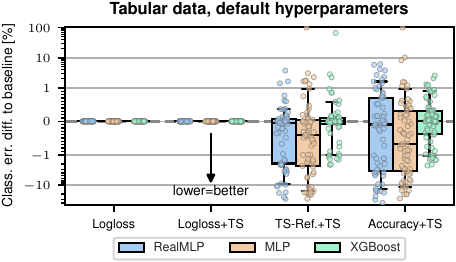}
    \end{minipage}
    \begin{minipage}{0.48\textwidth}
\centering
\includegraphics[width=\textwidth]{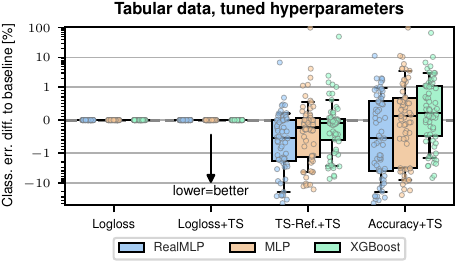}
    \end{minipage}

    \caption{\emph{Differences in downstream \emph{classification error} between logloss + TS and other procedures on tabular datasets.} Each dot represents one dataset from \citet{ye2024closer}, using only the 65 datasets with at least 10K samples. Percentages are clipped to $[-100, 100]$. Box-plots show the 10\%, 25\%, 50\%, 75\%, and 90\% quantiles.} \label{fig:tabular_accuracy}
\end{figure}

\begin{figure}[tb]
    \centering
    \begin{minipage}{0.48\textwidth}
\centering
\includegraphics[width=\textwidth]{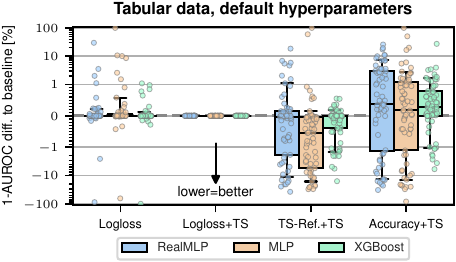}
    \end{minipage}
    \begin{minipage}{0.48\textwidth}
\centering
\includegraphics[width=\textwidth]{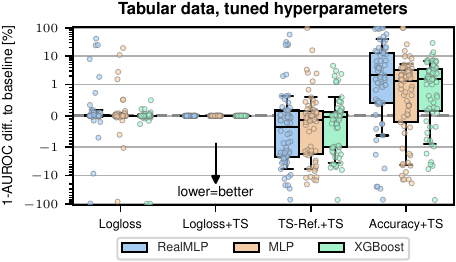}
    \end{minipage}

    \caption{\emph{Differences in downstream \emph{AUROC} (one-vs-rest for multi-class datasets) between logloss + TS and other procedures on tabular datasets.} Each dot represents one dataset from \citet{ye2024closer}, using only the 65 datasets with at least 10K samples. Percentages are clipped to $[-100, 100]$. Box-plots show the 10\%, 25\%, 50\%, 75\%, and 90\% quantiles.} \label{fig:tabular_auroc}
\end{figure}

\subsection{Stopping times}

\Cref{fig:stopping_times} shows the best epochs or iterations found for different metrics for the models with default parameters. We can see several tendencies across models:
\begin{itemize}
    \item Logloss stops first on average.
    \item Brier loss stops later than logloss, also when considering the loss after temperature scaling.
    \item Loss after temperature scaling stops later than loss before temperature scaling.
    \item Accuracy stops late and AUROC somewhere in the middle.
\end{itemize}

\begin{figure}
    \centering
    \begin{minipage}{0.48\textwidth}
        \subcaptionbox{RealMLP (default)}{
\centering
\includegraphics[width=0.96\textwidth]{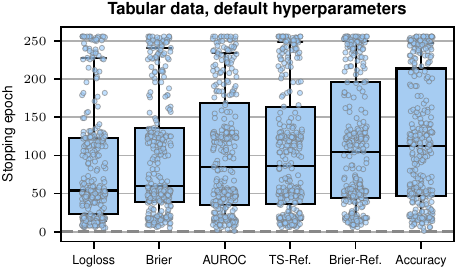}
}
    \end{minipage}
    \begin{minipage}{0.48\textwidth}
        \subcaptionbox{MLP (default)}{
\centering
\includegraphics[width=0.96\textwidth]{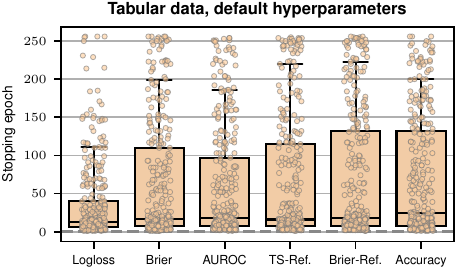}
}
    \end{minipage}
    \begin{minipage}{0.48\textwidth}
        \subcaptionbox{XGBoost (default)}{
\centering
\includegraphics[width=0.96\textwidth]{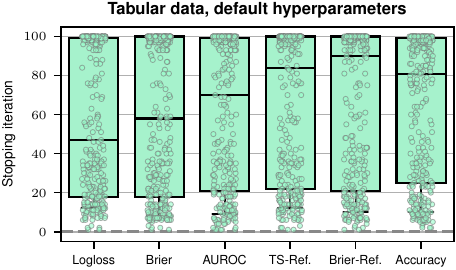}
}
    \end{minipage}
    \caption{\emph{Best epochs/iterations in default models for different stopping metrics.} Each point represents one model trained on one of five training-validation splits on one of the 65 datasets with at least 10K samples. Brier-Ref.\ refers to Brier loss after temperature scaling. AUROC uses the one-vs-rest version for multi-class datasets. The multiple clusters in the RealMLP stopping epochs are likely due to the use of a multi-cycle loss schedule.}
    \label{fig:stopping_times}
\end{figure}

\subsection{Methods}

From a single training of a method, we extract results for the best epoch (for NNs) or iteration (for XGBoost) with respect to all considered metrics. To achieve this, we always train for the maximum number of epochs/iterations and do not stop the training early based on one of the metrics.

\paragraph{Default parameters} For the MLP, we slightly simplify the default parameters of \cite{gorishniy2021revisiting} as shown in \Cref{table:mlp-d}. For XGBoost, we take the library default parameters. For RealMLP, we take the default parameters from \citet{holzmuller2024better} but deactivate label smoothing since it causes the loss to be non-proper.

\paragraph{Hyperparameter tuning} For hyperparameter tuning, we employ 30 steps of random search, comparable to \citet{mcelfresh2024neural}. Using random search allows us to train the 30 random models once and then pick the best one for each metric. For the MLP, we choose a small search space, shown in \Cref{table:mlp-hpo}, to cover the setting of a simple baseline. For XGBoost, we choose a typical, relatively large search space, shown in \Cref{table:xgb-hpo}. For RealMLP, we slightly modify the original search space from \citet{holzmuller2024better}, see \Cref{table:realmlp-hpo}.

\begin{table}[tb]
\centering
\small
\caption{Hyperparameter search space for XGBoost, adapted from \citet{grinsztajn2022tree}. We use the \texttt{hist} method, which is the new default in XGBoost 2.0 and supports native handling of categorical values, while the old \texttt{auto} method selection is not available in XGBoost 2.0. We disable early stopping as we want to select the best epoch for different metrics in the same training run. Some search spaces are adapted based on previous experiences.} \label{table:xgb-hpo}
\begin{tabular}{cc}
\toprule
Hyperparameter & Distribution \\
\midrule
tree\_method & hist \\
n\_estimators & 1000 \\
early\_stopping\_rounds & 1000 \\
max\_depth & UniformInt[1, 11] \\
learning\_rate & LogUniform[1e-3, 0.7] \\
subsample & Uniform[0.5, 1] \\
colsample\_bytree & Uniform[0.5, 1] \\
colsample\_bylevel & Uniform[0.5, 1] \\
min\_child\_weight & LogUniform[1e-5, 100] \\ 
alpha & LogUniform[1e-5, 5] \\
lambda & LogUniform[1e-5, 5] \\
gamma & LogUniform[1e-5, 5] \\
\bottomrule
\end{tabular}
\end{table}

\begin{table}[tb]
\small
\centering
\caption{Hyperparameter search space for RealMLP, adapted from \citet{holzmuller2024better}. Compared to the original search space, we make the option without label smoothing more likely since it optimizes a proper loss, and we insert a third option for weight decay.} \label{table:realmlp-hpo}
\begin{tabular}{ccc}
\toprule
Hyperparameter & Distribution \\
\midrule
Num.\ embedding type & Choice([None, PBLD, PL, PLR]) \\
Use scaling layer & Choice([True, False], p=[0.6, 0.4]) \\
Learning rate & LogUniform([2e-2, 3e-1]) \\
Dropout prob.\ & Choice([0.0, 0.15, 0.3], p=[0.3, 0.5, 0.2]) \\
Activation fct.\ & Choice([ReLU, SELU, Mish]) \\
Hidden layer sizes & Choice([[256, 256, 256], [64, 64, 64, 64, 64], [512]], p=[0.6, 0.2, 0.2]) \\
Weight decay & Choice([0.0, 2e-3, 2e-2]) \\
Num.\ emb.\ init std. & LogUniform([0.05, 0.5]) \\
Label smoothing $\varepsilon$ & Choice([0.0, 0.1]) \\
\bottomrule
\end{tabular}
\end{table}

\begin{table}[tb]
\centering
\small
\caption{Default hyperparameters for MLP, adapted from \cite{gorishniy2021revisiting}.} \label{table:mlp-d}
\begin{tabular}{ccc}
\toprule
Hyperparameter & Value \\
\midrule
Number of layers & 4 \\
Hidden size & 256 \\
Activation function & ReLU \\
Learning rate & 1e-3 \\
Dropout & 0.0 \\
Weight decay & 0.0 \\
Optimizer & AdamW \\
Preprocessing & Quantile transform + one-hot encoding \\
Batch size & 256 \\
Number of epochs & 256 \\
Learning rate schedule & constant \\
\bottomrule
\end{tabular}
\end{table}

\begin{table}[tb]
\centering
\small
\caption{Hyperparameter search space for MLP.} \label{table:mlp-hpo}
\begin{tabular}{ccc}
\toprule
Hyperparameter & Distribution \\
\midrule
Learning rate & LogUniform[1e-4, 1e-2] \\
Dropout & Choice([0.0, 0.1, 0.2, 0.3]) \\
Weight decay & Choice([0.0, 1e-5, 1e-4, 1e-3]) \\
\bottomrule
\end{tabular}
\end{table}

\subsection{Datasets}

We take the datasets from \citet{ye2024closer} and apply the following modifications:
\begin{itemize}
    \item Samples that contain missing values in numerical features are removed since not all methods natively support the handling of these values.
    \item Datasets that contain fewer than 1K samples after the previous steps are removed. The following datasets are removed: \texttt{vehicle}, \texttt{Pima\_Indians\_Diabetes\_Database}, \texttt{analcatdata\_authorship}, \texttt{mice\_protein\_expression}, and \texttt{eucalyptus}. The latter two datasets contained samples with missing numerical features. 
    \item Datasets with more than 100K samples are subsampled to 100K samples.
\end{itemize}

\end{appendices}

\end{document}